\documentclass[review]{elsarticle}

\usepackage[margin=3cm]{geometry}
\usepackage{epstopdf}
\usepackage[tight,footnotesize,rm,RM]{subfigure}
\usepackage{algorithm}
\usepackage[noend]{algorithmic}
\usepackage{url}
\usepackage{framed}
\usepackage[table]{xcolor}
\usepackage{hyperref}
\usepackage{amsthm,bm}
\usepackage{amsfonts,amssymb}
\usepackage{mathtools}
\usepackage{amssymb}
\usepackage{booktabs}
\usepackage{color}
\usepackage[font=footnotesize,skip=0pt]{caption}
\usepackage{multirow}
\usepackage{setspace}

\graphicspath{{./figs_arxiv/}}

\newtheorem{mydef}{Definition}[section]
\newtheorem{lemm}{Lemma}[section]

\newtheorem{proposition}{Proposition}
\newtheorem{property}{Property}[section]

\begin{document}

\begin{frontmatter}

\title{New advances in enumerative biclustering algorithms with online partitioning}

\author[myaddress1]{Rosana~Veroneze\corref{mycorrespondingauthor}}
\cortext[mycorrespondingauthor]{Corresponding author}
\ead{veroneze@dca.fee.unicamp.br / rveroneze@gmail.com}

\author[myaddress1]{Fernando~J.~Von~Zuben}

\address[myaddress1]{University of Campinas (DCA/FEEC), 400 Albert Einstein Street, Campinas, SP, Brazil}

\begin{abstract}
This paper further extends RIn-Close\_CVC, a biclustering algorithm capable of performing an efficient, complete, correct and non-redundant enumeration of maximal biclusters with constant values on columns in numerical datasets. By avoiding a priori partitioning and itemization of the dataset, RIn-Close\_CVC implements an online partitioning, which is demonstrated here to guide to more informative biclustering results. The improved algorithm is called RIn-Close\_CVC3, keeps those attractive properties of RIn-Close\_CVC, as formally proved here, and is characterized by: a drastic reduction in memory usage; a consistent gain in runtime; additional ability to handle datasets with missing values; and additional ability to operate with attributes characterized by distinct distributions or even mixed data types. The experimental results include synthetic and real-world datasets used to perform scalability and sensitivity analyses. As a practical case study, a parsimonious set of relevant and interpretable mixed-attribute-type rules is obtained in the context of supervised descriptive pattern mining.
\end{abstract}


\end{frontmatter}

\section{Introduction}
\label{sec:intro}

Biclustering is a powerful data analysis technique that, essentially,  consists in discovering submatrices indicating some regularities in a data matrix. It has been successfully applied in various domains, such as analysis of gene expression data, collaborative filtering, text mining, treatment of missing data, and dimensionality reduction. However, due to the complexity of the biclustering problems, most of the proposed biclustering algorithms are heuristic-based, leading to suboptimal solutions.

Nonetheless, in the areas of Formal Concept Analysis (FCA), Frequent Pattern Mining (FPM), and graph theory (GT), we have plenty of algorithms for enumerating all maximal biclusters with constant values (CTV) in a binary dataset. These maximal CTV biclusters are called formal concepts in FCA, closed frequent itemsets in FPM (being more specific, a closed frequent itemset corresponds to the column-set of a bicluster), and maximal bicliques in graph theory. Some examples of these enumerative algorithms are: Close-by-One (CbO) \cite{Kuznetsov1999}, In-Close \cite{Andrews2009}, In-Close2 \cite{Andrews2011}, FCbO \cite{Krajca2010,OutrataEtAl2012}, and LCM \cite{UnoEtAL2004}. Their enumeration process is characterized by being:

\begin{enumerate}
\item Efficient: it has time complexity linear in the number of biclusters and polynomial in the input size.
\item Complete: it finds all maximal biclusters in a dataset. Remark that a complete enumeration incorporates the results produced by any other biclustering solution (given the same restrictions of internal consistency and size). So, such biclustering solution is at least of equal quality when compared with the solution provided by any other contender. 
\item Correct: all biclusters obey the user-defined measure of internal consistency. For instance, in the case of the aforementioned enumerative algorithms, all biclusters are submatrices of ones.
\item Non-redundant: all biclusters are maximal and it does not enumerate the same maximal bicluster more than once. It is a very important property because the number of biclusters produced from a dataset can be very large. So, it is useful to identify the smallest representative set of biclusters from which all other biclusters can be derived \cite{TanEtAl2005}. The set of all maximal biclusters is necessary and sufficient to capture all the information about the biclusters, and has a much smaller cardinality than the set of all attainable biclusters \cite{Zaki2000}. It is important to note that the algorithm must have a smart solution to avoid redundancy, otherwise it will not be efficient. For instance, a procedure to be avoided is to check if a new bicluster is not redundant by comparing with all previously mined biclusters.
\end{enumerate}

Recently, Veroneze \textit{et al.} \cite{VeronezeEtAl2017} proposed a family of algorithms, called RIn-Close, also exhibiting these four key properties when enumerating biclusters directly in numerical (not only binary, but also integer or real-valued) data matrices. It may be considered a significant achievement, given that, before the RIn-Close family of algorithms, finding biclusters in numerical data matrices was accomplished by algorithms not exhibiting those four properties, or by a priori partitioning the numerical matrix followed by \textit{itemization}\footnote{Itemization is a terminology used in \cite{HenriquesEtAl2015}. In the FCA literature, a more common terminology is \textit{scaling} \cite{Ganter1997}.}. RIn-Close algorithms, in turn, perform an online partitioning of the numerical attributes, attenuating the information loss implied by any a priori partitioning operation.

\subsection{Contributions}

Even though Veroneze \textit{et al.} \cite{VeronezeEtAl2017} have proposed a number of contributions to the enumeration of biclusters in numerical datasets, there is still room for improvements that would bring many benefits to the users of RIn-Close algorithms. We focus here on the RIn-Close\_CVC algorithm \cite{VeronezeEtAl2017}, which enumerates biclusters with constant values on columns (CVC). This algorithm has already been applied with great success in gene ontology enrichment analysis \cite{Veroneze2016}, screening and identification of biomarkers \cite{Veroneze2016}, and identification of discriminative patterns in labeled datasets \cite{VeronezeVonZuben2017report}. We propose here a new version of RIn-Close\_CVC \cite{VeronezeEtAl2017}, named RIn-Close\_CVC3\footnote{We are naming our new algorithm \textit{RIn-Close\_CVC3} because we have used the name \textit{RIn-Close\_CVC2} \cite{VeronezeVonZuben2018} in an arXiv report where we first explained our strategy to remove the necessity of the symbol table present in RIn-Close\_CVC \cite{VeronezeEtAl2017}. That report does not encompass the other novelties presented here. We use the RIn-Close\_CVC2 nomenclature in our experimental section to highlight the breakthrough promoted by the elimination of the symbol table.}, with the following novelties:
\begin{itemize}
\item RIn-Close\_CVC \cite{VeronezeEtAl2017} is efficient in terms of runtime, but has a high computational cost in terms of memory usage. It must keep a symbol table in memory, whose keys are the row-sets of each found bicluster, to prevent a maximal bicluster to be found more than once. RIn-Close\_CVC3 does not use a symbol table to prevent redundant biclusters because the redundancy is avoided by further exploring the lexicographic order \cite{Ganter1984}. Nevertheless, this new version keeps the four aforementioned key properties of an enumerative biclustering algorithm. The experimental results show that in addition to a large reduction in memory usage, the new algorithm also brings an overall significant runtime gain \cite{VeronezeVonZuben2018}. This achievement opens new possibilities of application of enumerative algorithms with online partitioning. 
\item RIn-Close\_CVC \cite{VeronezeEtAl2017} is conceptually based on the FCA algorithm In-Close2 \cite{Andrews2011}, but since then we have new versions of the baseline algorithm In-Close2: In-Close3 \cite{Andrews2015}, In-Close4 \cite{Andrews2017}, and In-Close5 \cite{Andrews2018}.
In-Close2 was used as the starting point for In-Close4\footnote{In-Close3 is not competitive and is not going to influence our new proposal.}, which achieves a better performance than In-Close2 by making use of empty intersections (or intersections with size less than a user-defined parameter) in the computation of formal concepts. In-Close5 improves In-Close4 by adding the feature of inheriting canonicity test failures. It is done in a simpler way when compared to what is done in previous existing algorithms, such as FCbO. Here, we generalize the contributions proposed in \cite{Andrews2017, Andrews2018} to incorporate them into RIn-Close\_CVC3, as will be explained along the manuscript. It promotes  an even better runtime, specially for datasets with many attributes.
\item It handles datasets with missing values by looking for biclusters exclusively in regions of the dataset with no missing values \cite{Veroneze2016}. In this way, the sparser the matrix, the faster the enumeration \cite{Veroneze2016}. This new ability can also be explored to avoid uninteresting biclusters, speeding up the enumeration process (for instance, if the user is not interested in biclusters with some specific values, these values can be interpreted as missing values). 
\item It handles datasets characterized by attributes with very distinct distributions, and even mixed data types. For instance, RIn-Close\_CVC3 can enumerate biclusters from a mixed-attribute dataset, exhibiting numerical (discrete or continuous) and categorical (ordinal or nominal) attributes concomitantly \cite{VeronezeVonZuben2017report}.
\end{itemize}


Other contributions of this work are:
\begin{itemize}
\item By the first time, we formally demonstrated that RIn-Close\_CVC and its new version, RIn-Close\_CVC3, have the four aforementioned key properties.
\item More general types of biclusters, such as biclusters with coherent values (CHV) and order-preserving submatrices (OPSM), can be mined from CVC biclusters, as stated in \cite{HenriquesMadeira2014, Veroneze2016}. So, advances in the enumeration of CVC biclusters bring immediate benefits to the enumeration of other types of biclusters.
\item We provide a didactic example of how to handle mixed-attribute datasets using our new algorithm RIn-Close\_CVC3.
\item Following the formalization of enumerative biclustering approaches for non-binary datasets provided by Henriques et al. \cite{HenriquesEtAl2015}, which divides the proposals into two classes (a priori and online partitioning), we by the first time provide an extensive didactic example that highlights strengths and peculiarities of the two classes of approaches. The conclusions that we draw from this example may help many analysts to choose the most indicated approach for their applications.
\item We further explore the strong connection between biclustering and FPM and show that a bicluster provides all necessary information to build a (quantitative) itemset, which is a component of a (quantitative) association rule \cite{AgrawalEtAl1993, SrikantAgrawal1996} or the antecedent of a (quantitative) class association rule \cite{LiuEtAl1998, VeronezeVonZuben2017report}. Furthermore, rules are simple and interpretative formats to present biclusters to the user, indicating that enumerative biclustering techniques can bring great benefits to \textit{supervised descriptive pattern mining} \cite{NovakEtAl2009, VenturaLuna2018sdpm}.
\end{itemize}

Our RIn-Close implementation is available at \url{https://github.com/rveroneze/rinclose}. We have updated all algorithms of the RIn-Close family to handle missing values and to incorporate the contributions proposed in \cite{Andrews2017, Andrews2018}.

\subsection{Structure of the paper}

The remainder of the paper is organized as follows. Section~\ref{sec:bic} introduces definitions and a mathematical notation for biclustering, and gives a short view of related areas. Section~\ref{sec:RelWorks} is devoted to connected works. Section~\ref{sec:recap} reviews the algorithms In-Close5 \cite{Andrews2018} and RIn-Close\_CVC \cite{VeronezeEtAl2017}. Section~\ref{sec:rinclosecvc3} presents and analyzes the main contribution of this paper: the algorithm RIn-Close\_CVC3. Experimental results are discussed in Section~\ref{sec:expresults}, and we conclude in Section~\ref{sec:conclusion}. \ref{sec:appendix_ma} illustrates how to mine biclusters in a mixed-attribute dataset, and \ref{sec:appendix_class1x2} compares the enumerative biclustering approaches with a priori and online partitioning.

\section{Biclustering}
\label{sec:bic}
The formalism used here to describe a bicluster is based on \cite{VeronezeEtAl2017}. Let $\mathbf{A}_{n \times m}$ be a data matrix with the row index set $X = \left \{ 1, 2,..., n \right \}$ and the column index set $Y = \left \{ 1, 2, ...,m \right \}$. Each row represents an object, and each column represents an attribute. Each element $a_{ij} \in \mathbf{A}$ holds the relationship between object $i$ and attribute $j$. We use $(X,Y)$ to denote the entire matrix $\mathbf{A}$. Considering that $I \subseteq X$ and $J \subseteq Y$, $\mathbf{A}_{IJ} = (I, J)$ corresponds to the submatrix of $\mathbf{A}$ with the row index subset $I$ and column index subset $J$.

\begin{mydef}
A bicluster is a submatrix $(I,J)$ of the data matrix $\mathbf{A}_{n \times m}$ such that the rows in the index subset $I = \left \{ i_1,..., i_k \right \}$ ($I \subseteq X$ and $k \leq n$) exhibit a consistent pattern across the columns in the index subset $J = \left \{ j_1,..., j_s \right \}$ ($J \subseteq Y$ and $s \leq m$). Simultaneously, the columns of the same submatrix $(I,J)$ also exhibit a consistent pattern across the rows.
\label{def:bic}
\end{mydef}

Thus, a bicluster $(I,J)$ is a $k \times s$ submatrix of the matrix $\mathbf{A}$, not necessarily with contiguous rows and columns, such that it meets a certain internal consistency criterion \cite{MadeiraOliveira2004}. A biclustering algorithm looks for a set of biclusters $\mathfrak{B} = (I_l, J_l)$, $l = 1, ..., q$, with the total number of biclusters, $q$, being dependent on the characteristics of the selected biclustering algorithm, on the constraints imposed (for instance, size and consistency criterion), and on the behavior of the dataset being analyzed. The coverage of a biclustering solution $\mathfrak{B}$ is given by $cov(\mathfrak{B}) = \left | \bigcup_{(I_l,J_l)} I_l \times J_l \right |.$ The coverage metric indicates the portions of the dataset that is explored by a biclustering solution. The biclusters extract and summarize information from the covered areas.

Depending on the pattern of internal consistency, there are four major types of biclusters \cite{MadeiraOliveira2004}: ($i$) biclusters with constant values (CTV), ($ii$) biclusters with constant values on columns (CVC) or rows (CVR), ($iii$) biclusters with coherent values (CHV), and ($iv$) biclusters with coherent evolutions (CHE). There are many subtypes of CHE biclusters, with the order-preserving submatrix (OPSM) biclusters being the most popular among them.

Since we propose here a new version of RIn-Close\_CVC, we are focusing on CVC biclusters. Also, as CTV biclusters are a special case of CVC biclusters, we will give these two definitions in what follows. See \cite{MadeiraOliveira2004, VeronezeEtAl2017} for the definitions and examples of all types of biclusters. In our definitions, a user-defined parameter $\epsilon \geq 0$ determines the maximum perturbation (residue) allowed in a \textbf{correct} bicluster.

\begin{mydef}[CTV biclusters]
A \emph{CTV bicluster} is a submatrix $(I, J)$ of matrix $\mathbf{A}_{n \times m}$ such that

\vspace*{-0.6cm}

\begin{equation}
\max_{i \in I, j \in J} (a_{ij}) - \min_{i \in I, j \in J} (a_{ij}) \leq \epsilon \text{, with } \epsilon \ge 0.
\label{eq:ctvbic}
\end{equation}
\label{def:ctvbic}
\end{mydef}

\vspace*{-1cm}

\begin{mydef}[CVC biclusters]
A \emph{CVC bicluster} is a submatrix $(I, J)$ of matrix $\mathbf{A}_{n \times m}$ such that

\vspace*{-0.6cm}

\begin{equation}
\max_{i \in I} (a_{ij}) - \min_{i \in I} (a_{ij}) \leq \epsilon, \forall j \in J \text{, with } \epsilon \ge 0.
\label{eq:cvcbic}
\end{equation}
\label{def:cvcbic}
\end{mydef}

\vspace*{-1cm}

Biclusters that can be mined using $\epsilon = 0$ are named here \textit{perfect biclusters}, otherwise they are named \textit{perturbed biclusters}. The RIn-Close family of algorithms \cite{VeronezeEtAl2017} has specialized algorithms for mining perfect biclusters that are simpler and faster than their versions for mining perturbed biclusters.

The definition of a CVR bicluster is the equivalent transpose of the definition of a CVC bicluster. So, we can mine CVR biclusters by transposing the original data matrix and using an algorithm to mine CVC biclusters. Note that CVC and CVR biclusters 
are generalizations of CTV biclusters, as demonstrated in Lemma \ref{lem:bicCTVtoCVC}.

\begin{lemm}
A CTV bicluster with residue $\epsilon$ is a CVC (CVR) bicluster with residue $\epsilon'$ such that $\epsilon' \leq \epsilon$.
\label{lem:bicCTVtoCVC}
\end{lemm}
\begin{proof}
If a CTV bicluster has residue $\epsilon$, it means that the maximum pairwise variation between its elements is $\epsilon$. Given that the elements are not restricted to be part of the same column, the maximum variation per column, considering all the columns, may be less than $\epsilon$. Therefore, $\epsilon' \leq \epsilon$.

The proof for a CVR bicluster is equivalent.
\end{proof}

\subsection{Maximality and monotonicity}
\label{subsec:bicMaxProp}

\begin{mydef}[Maximal bicluster]
Let $(I,J)$ be a correct bicluster of any type, it is called a \emph{maximal bicluster} if and only if:
\begin{itemize}
\item $\forall x \in X \setminus I$, $(I \cup \{x\}, J)$ is not a \emph{correct} bicluster (there is no additional row that could be included in the bicluster without violating the desired pattern of internal consistency), and
\item $\forall y \in Y \setminus J$, $(I, J \cup \{y\})$ is not a \emph{correct} bicluster (the same for any additional column).
\end{itemize}
\label{def:maximal}
\end{mydef}

\begin{property}[Anti-Monotonicity]
Let $(I,J)$ be a correct bicluster. Any submatrix $(I', J')$, where $I' \subseteq I$ and $J' \subseteq J$, is also a \emph{correct} bicluster.
\end{property}

\begin{property}[Monotonicity]
Let $(I,J)$ be a correct and maximal bicluster. Any supermatrix $(I',J')$, where $I' \times J' \supset I \times J$, is not a \emph{correct} bicluster.
\end{property}

Usually, the efficient enumerative algorithms of FCA and FPM areas are based on the (anti-~)~monotonicity property \cite{Besson2007}, as well as the RIn-Close family of algorithms. In fact, we do not know any efficient enumerative biclustering algorithm that is not based on these properties. Our given definitions of CTV and CVC biclusters are in accordance with these two properties.

\subsection{Generalization of the definition of a CVC bicluster}
\label{ssec:ma}

We propose a generalization of the definition of a CVC bicluster \cite{VeronezeEtAl2017} that allows us to handle datasets characterized by attributes with distinct distributions, and even attributes with distinct data types:

\begin{mydef}[CVC biclusters]
A \emph{CVC bicluster} is a submatrix $(I, J)$ of matrix $\mathbf{A}_{n \times m}$ such that

\vspace*{-0.6cm}

\begin{equation}
\max_{i \in I} (a_{ij}) - \min_{i \in I} (a_{ij}) \leq \epsilon_j, \forall j \in J,
\label{eq:cvcbic2}
\end{equation}

\vspace*{-0.25cm}

\noindent where $\epsilon_j \geq 0$ is the user-defined maximum allowed perturbation for attribute $j$.
\label{def:cvcbic2}
\end{mydef}

In this way, attributes with different distributions, and even attributes of different data types can be jointly treated. An example of how to mine biclusters in a mixed-attribute dataset is provided in \ref{sec:appendix_ma}. Remarkably, this new definition of CVC biclusters also meets the (anti)-monotonicity property, which is fundamental for enumerative algorithms.

\subsection{Biclustering and formal concept analysis}
\label{subsec:fca}

Let $\mathbf{A}_{n \times m} = (X,Y)$ be a binary data matrix. In FCA, any binary matrix can be referred as a \emph{formal context} in the form:

\begin{mydef}[Formal Context]
A \emph{formal context} is a triple $(X, Y, Z)$ of two sets $X$ and $Y$, and a binary relation $Z \subseteq X \times Y$.
\end{mydef}

For a subset $I \subseteq X$, we define $I^{\uparrow} = \{y \in Y| a_{iy} = 1, \forall i \in I\}$ as the set of attributes common to all the objects in $I$. Similarly, for a subset $J \subseteq Y$, we define $J^{\downarrow} = \{x \in X| a_{xj} = 1, \forall j \in J\}$ as the set of objects common to all the attributes in $J$.

\begin{mydef}[Formal Concept]
A formal concept of the formal context $(X, Y, Z)$ is a pair $(I,J)$ with $I \subseteq X$, $J \subseteq Y$, and such that $I^{\uparrow} = J$, and $J^{\downarrow}=I$.
\end{mydef}

The row-set $I$ of a formal concept $(I,J)$ is called \emph{extent}, and the column-set $J$ is called \emph{intent}. By the definition, though many subsets $I$ can generate the same subset $J$, only the largest (closed) subset $I$ is part of a formal concept. The same reasoning is valid for $J$. Therefore, a formal concept is a maximal CTV bicluster of 1's in a binary data matrix.

Formal concepts are partially ordered by $(I_1, J_1) \leq (I_2, J_2) \Leftrightarrow I_1 \subseteq I_2 (\Leftrightarrow J_2 \subseteq J_1)$. With respect to this partial order, the set of all formal concepts forms a complete lattice called the \emph{concept lattice} of the formal context $(X,Y,Z)$.

\subsection{Biclustering and association mining}

The association mining \cite{ceglarEtAl2006} problem is divided in two sub-problems: (\emph{i}) the frequent itemset (pattern) mining problem, and (\emph{ii}) the problem of mining the association rules from these itemsets.

Let $\mathbf{A}_{n \times m} = (X,Y)$ be a binary matrix with each row representing a \emph{transaction}, and each column representing an \emph{item}.

\begin{mydef}
A subset $J = \left \{ j_1,..., j_s \right \} \subseteq Y$ is called an \emph{itemset}.
\end{mydef}

\noindent The support of an itemset $J$ is given by $sup(J) = |J^{\downarrow}|$.

The problem of mining all frequent itemsets can be described as follows: determine all subsets $J \subseteq Y$ such that the support of $J$ is above a user-defined parameter. Apriori \cite{AgrawalEtAl1993} is an example of algorithm that performs this task.


An option to reduce the computational cost of the frequent pattern mining problem without loss of information is to mine only the \emph{closed frequent itemsets}. A frequent itemset $J$ is called \emph{closed} if there exists no superset $H \supset J$ with $H^{\downarrow} = J^{\downarrow}$, i.e., a closed frequent itemset is the intent of a formal concept. The closed frequent itemsets are also called \emph{frequent concept intents}. For any itemset $J$, its concept intent is given by $J^{\downarrow\uparrow}$. Remarkably, a concept lattice contains all necessary information to derive the support of all (frequent) itemsets \cite{LakhalStumme2005}. So, the set of closed frequent itemsets uniquely determines the exact frequency of all frequent itemsets, and it can be orders of magnitude smaller than the set of all frequent itemsets \cite{ZakiEtAL2002}. Therefore, this approach drastically reduces the number of rules that have to be presented to the user, without any information loss \cite{LakhalStumme2005}. 

\begin{mydef}
	An \emph{association rule} (AR) is an expression of the form $J \Rightarrow H$, where $J$ and $H$ are itemsets, and $H \cap J = \emptyset$.
\end{mydef}


A user can be interested in a more specific set of association rules, where the consequences of the rules describe a target variable. These rules are known as class association rules (CARs).

\begin{mydef}
	A \emph{class association rule} (CAR) is an expression of the form $J \Rightarrow c$, where $J$ is an itemset and $c$ is a class label.
\end{mydef}

As far as we known, Srikant \& Agrawal \cite{SrikantAgrawal1996} provided the first attempt to address the problem of mining quantitative ARs. Their proposal relies on the partitioning of the numerical (quantitative) attributes and the usage of an Apriori-based algorithm. Here, we will generalize the concepts of itemsets, ARs, and CARs as follows.

\begin{mydef}
A \emph{quantitative itemset} $\mathfrak{J}$ is a set of attributes and their domain of interest, i.e., $\mathfrak{J} = \{j_1 \in D_1, j_2 \in D_2, ..., j_s \in D_s \}$, where $j_\mathfrak{i}$ is an attribute and $D_\mathfrak{i}$ is its domain of interest. If $j_\mathfrak{i}$ is a discrete attribute, $D_\mathfrak{i}$ is a finite set of values; if $j_\mathfrak{i}$ is a continuous attribute, $D_\mathfrak{i}$ is an interval.
\end{mydef}

\begin{mydef}
A \emph{quantitative association rule} (QAR) is an expression of the form $\mathfrak{J} \Rightarrow \mathfrak{H}$, where $\mathfrak{J}$ and $\mathfrak{H}$ are quantitative itemsets, and the intersection of the attributes of $\mathfrak{J}$ and $\mathfrak{H}$ is empty.
\label{def:qar}
\end{mydef}

\begin{mydef}
A \emph{quantitative class association rule} (QCAR) is an expression of the form $\mathfrak{J} \Rightarrow c$, where $\mathfrak{J}$ is a quantitative itemset and $c$ is a class label.
\end{mydef}

Let $J$ be equal to the column indexes of the attributes in $\mathfrak{J}$.  Let $I$ be the set of objects that meets the requirement imposed by the quantitative itemset $\mathfrak{J}$: $I = \{i \in X| a_{ij} \in D, \; \; \forall (j \in D) \in \mathfrak{J} \}$. Thus, the pair $(I,J)$ is a CVC bicluster. So, a CVC bicluster $(I,J)$ provides all the necessary information to build a quantitative itemset (that is a component of a QAR or of a QCAR), and vice versa.

We highlight that quantitative association rules are generally divided into two classes in the literature: frequent rules and distributional rules \cite{zhu2009}. Our definitions are based on \emph{frequent rules} because it is the case that has a direct relation with the biclustering problem. See \cite{zhu2009} for more details about distributional rules.

\subsection{Biclustering and supervised descriptive pattern mining}
Data analysis is generally divided into two main groups: predictive and descriptive tasks. Predictive tasks involve models used to estimate a target variable. Descriptive tasks aim at finding comprehensive patterns on unlabeled data. However, several application domains sometimes require that the detected patterns somehow match previously intended patterns, giving rise to the concept of \textit{supervised descriptive pattern mining} \cite{NovakEtAl2009, VenturaLuna2018sdpm} (also called \textit{supervised descriptive rule discovery}). 


Supervised descriptive discovery gathers the following main tasks \cite{VenturaLuna2018sdpm}: contrast set mining, emerging pattern mining, subgroup discovery, class association rules mining, and exceptional model mining.

Given the strong connection between biclustering and FPM, with the possibility of mining QCARs by means of biclustering, we claim that enumerative biclustering algorithms can be successfully applied in supervised descriptive pattern mining. We will provide some examples in Section~\ref{sec:expresults} that illustrate the ability of our approach in finding rules that properly describe the variable of interest.

We also remark that, once it was properly characterized \cite{LiuEtAl1998}, CARs have gained an increasing attention due to their descriptive power and many research studies have considered these rules not only for descriptive purposes but also to form accurate classifiers \cite{VenturaLuna2018sdpm}, known as associative classifiers.

\section{Related works}
\label{sec:RelWorks}

Given that biclustering can be interpreted as a challenging combinatorial optimization task, most of the existing algorithms are heuristic-based, potentially producing suboptimal biclustering solutions \cite{HenriquesEtAl2015}. On the other hand, enumerative approaches based on FCA, FPM, or GT promote an exhaustive yet efficient search. As already mentioned, in these areas we have plenty of algorithms for enumerating all maximal CTV biclusters of ones in a binary dataset. These algorithms are also explored for the enumeration of biclusters in non-binary datasets. In their survey, Henriques et al. \cite{HenriquesEtAl2015} divided the enumerative biclustering approaches in two classes: with a priori partitioning (named here Class1) and with online partitioning (named here Class2). \ref{sec:appendix_class1x2} provides an extensive didactic example that highlights strengths and peculiarities of the two classes of approaches.

Class1 relies on the usage of traditional FPM / FCA / GT enumerative algorithms by means of preprocessing the original dataset in order to obtain an \textit{itemized dataset} (see \ref{sec:appendix_class1x2} for details), which is essentially a binary dataset. Thus, any traditional enumerative algorithm can be used to mine the biclusters. Class2 is composed of enumerative biclustering algorithms that handles directly non-binary datasets by extending / generalizing traditional methods based on the (anti-)monotonicity property. This is the case of the RIn-Close family of biclustering algorithms.

The main drawback of Class1 approaches is that even if one is able to perform an optimized partitioning, which is generally not the case and tends to guide to a combinatorial explosion of possibilities, there will be loss of information, simply because Class1 algorithms resort to an a priori partitioning operation. Any a priori and feasible partitioning involves loss of information and can not perform what an online partitioning procedure would produce. Working directly on the numerical dataset, proposals from Class2, such as RIn-Close\_CVC, perform an online checking for correctness at the time the biclusters are being built.

Hence, Class2 approaches are more flexible than Class1 approaches. Besides, according to the convenience of the user, pre-processing (such as normalization, scaling, or partitioning) of any attribute is fully admissible for the Class2 approaches as well. Nonetheless, \textbf{the extension / generalization of traditional FPM / FCA algorithms implies an extra computational cost, and reducing this burden is one of the main purposes here}. Although we have many efficient enumerative algorithms in traditional FCA / FPM, this is yet not true for online partitioning. In what follows, we list all known Class2 algorithms for enumerating CVC biclusters that are direct competitors of RIn-Close\_CVC algorithms.

RAP \cite{PandeyEtAl2009} is an algorithm to mine CVR (CVC) biclusters that is based on Apriori \cite{AgrawalSrikant1994}. Apriori-based algorithms mine all frequent itemsets, not only the closed frequent itemsets. We emphasize that the usage of algorithms for mining closed frequent itemsets instead of frequent itemsets drastically reduces the number of biclusters due to non-redundancy, and without any information loss. The authors of RAP did not describe their strategy to avoid redundancy, but we conjecture that the best that can be done is a pairwise comparison of biclusters with $k$ and $k+1$ columns or checking the maximality of each bicluster in the original dataset (by trying to add more columns).

ET-bicluster \cite{GuptaEtAl2011} extends the previous approach, RAP, to discover noisy CVR (CVC) biclusters, although an exhaustive enumeration of biclusters is not guaranteed. ET-bicluster also adopts an Apriori-like strategy, generating $(k+1)$-level biclusters from $k$-level biclusters.

In \cite{CodocedoNapoli2014a, CodocedoNapoli2014b}, the authors proposed a method based on \emph{partition pattern structures} (PPS) \cite{BaixeriesEtAl2014}. 
While FCA algorithms mine formal concepts, PPS algorithms mines \textit{partition pattern concepts} (\textit{pp-concepts}). In \cite{CodocedoNapoli2014a}, the CVC biclusters extracted from the pp-concepts are \textbf{perfect}. The proposal in \cite{CodocedoNapoli2014b} generalizes it by using equivalence relations \cite{Besson2007}, so it can mine perturbed CVC biclusters as well. However, the authors imposed that the partitions should be disjoint. Other important aspect is that, while a pp-concept $(B, d)$ (where $B$ is the columns and $d$ is the partitions of the pp-concept) is maximal by definition, a bicluster $(p, B)$ (with $p \in d$) mined from the pp-concept $(B, d)$ can be non-maximal.

In \cite{KaytoueEtAl2014}, the authors revisited the proposals of mining CVC biclusters using \emph{interval pattern structures} (IPS) \cite{Kaytoue2011} and PPS \cite{CodocedoNapoli2014a, CodocedoNapoli2014b} (in this work, the partitions are not restricted to be disjoint). They also proposed an approach based on Triadic Concept Analysis (TCA) \cite{Lehmann1995}. Again we have that interval pattern concepts and triadic concepts are maximal by definition, but not all CVC biclusters extracted from them are maximal.

So, the difference among PPS, IPS, and TCA techniques is not given by the final number of maximal biclusters, but by the number of concepts found and their post-processing complexity to extract the maximal biclusters from them \cite{KaytoueEtAl2014}. For instance, according to Codocedo and Napoli \cite{CodocedoNapoli2014b}, their experiments showed that less than 20\% of the pp-concepts within the pp-lattice actually hold a maximal bicluster. Thus, the efficiency of these proposals depends on the efficiency of the additional step of redundancy suppression. According to Kaytoue \textit{et al.} \cite{KaytoueEtAl2014}, the post-processing of TCA is linear w.r.t. the number of triadic concepts found, while for IPS is linear w.r.t. the number of interval pattern concepts times the number of columns of the dataset squared, and for PPS is linear w.r.t. the number of super-sub pp-concept relations in the tolerance block pattern concept lattice.

However, in \cite{CodocedoNapoli2014a, CodocedoNapoli2014b}, the authors proposed a much faster strategy to suppress redundancy in PPS by breaking the pp-lattice during its construction. They show that it is possible to know if a bicluster $(p, B)$ mined from the pp-concept $(B, d)$ is maximal by means of looking to the super pp-concepts of the pp-concept $(B, d)$. It is still necessary to mine the super-sub pp-concepts in order to promote the verifications, so the overall worst-case of this proposal will not have time complexity linear in the number of maximal biclusters. Nevertheless, according to Kaytoue \textit{et al.} \cite{KaytoueEtAl2014}, this strategy cuts in half the runtime, being the fastest one among the three approaches. Therefore, as PPS is the fastest competitor, we will compare its performance against RIn-Close\_CVC algorithms in the experimental section of this work. Still according to Kaytoue \textit{et al.} \cite{KaytoueEtAl2014}, similar improvements for IPS and TCA could also be implemented but are still a matter of research.

Table~\ref{tab:comparison_cvc} summarizes the above considerations. As far as we known, RIn-Close\_CVC algorithms are the only ones that exhibit the four key properties (efficiency, completeness, correctness and non-redundancy) when enumerating CVC biclusters with a maximum perturbation $\epsilon$.

\begin{table}[!htb]
\centering
\small
\caption{Comparison between RIn-Close\_CVC and its Class2 competitors.}
\begin{tabular}{ccccc}
\toprule
& Complete	& Correct & Non-Redundant & Efficient \\
\midrule
RAP \cite{PandeyEtAl2009} & $\checkmark$ & $\checkmark$ &  &  \\
ET-bicluster \cite{GuptaEtAl2011} & & $\checkmark$ &  &  \\
PPS \cite{KaytoueEtAl2014} & $\checkmark$ & $\checkmark$ & $\checkmark$ &  \\
IPS \cite{KaytoueEtAl2014} & $\checkmark$ & $\checkmark$  & $\checkmark$ &  \\
TCA \cite{KaytoueEtAl2014} & $\checkmark$ & $\checkmark$  & $\checkmark$ &  \\
RIn-Close\_CVC & $\checkmark$	& $\checkmark$ & $\checkmark$ & $\checkmark$ \\
\bottomrule
\end{tabular}
\label{tab:comparison_cvc}
\end{table}

\subsection{Important issues of enumerative biclustering}

Some important aspects of enumerative biclustering are out of the scope of this work, so we strongly recommend the reading of the survey of Henriques et al. \cite{HenriquesEtAl2015} and our previous work \cite{VeronezeEtAl2017} for a more complete overview of this area. Here, we will talk briefly about some of these issues and indicate some additional references.

It is well known in FCA and FPM that the number of biclusters in an enumerative solution can be huge, even using condensed representations such as maximal biclusters. Besides, a bicluster may have a high internal consistency but not be relevant for some user's applications. In the literature, there are many general metrics that can help the user to filter a biclustering solution, for instance see \cite{HenriquesMadeira2018BSig, HortaCampello2014, KuznetsovEtAl2018, LeeEtAl2015, MartinezEtAl2014, ZakiMeira2014, Zimmermann2015}, but also the user can incorporate domain knowledge to the filter \cite{HenriquesMadeira2016}. Moreover, a biclustering solution can have many biclusters with large overlap between them, which makes it even harder to pick the ones that matter. Noise is one of the responsible for fragmenting a single bicluster into several ones with high overlapping \cite{ZhaoZaki2005}. Proposals for post-processing the biclustering solution based on the overlap between the biclusters can be found in \cite{OliveiraEtAl2015,YanEtAl2005,ZhaoZaki2005}. We highlight that it is also possible to mitigate the noise effect on the enumeration by preprocessing the input data, see \cite{HenriquesMadeira2014} for details.

Choosing the parameters of enumerative biclustering algorithms, with a priori or online partitioning, may not be an easy task, especially for new users. For this reason, BicPAM \cite{HenriquesMadeira2014} makes available a data-driven parameterization that relies on estimation procedures and on convergence criteria based on thresholds such as coverage, thus providing a good background for the parameterization of enumerative biclustering.

When focused on handling missing values, our proposal is essentially the same as that one proposed in BicPAM \cite{HenriquesMadeira2014} when considering the \textit{restrictive} alternative. As BicPAM works with a priori partitioning, in the less constrained alternatives, Henriques and Madeira also proposed the usage of \textit{additional items} handled according to a level of relaxation defined by the user. Thus, the biclusters can carry estimations of the missing values, but we can have biclusters with different estimations for the same missing value. Anyway, it is another possible approach to handle the missing values. Naturally, standard techniques, such as imputation, can also be used.

\section{Recapping principles from biclustering searches}
\label{sec:recap}

In this section, we will review In-Close5 \cite{Andrews2018} and RIn-Close\_CVC \cite{VeronezeEtAl2017}, the starting points for RIn-Close\_CVC3. As In-Close5 \cite{Andrews2018} incorporates the contribution of In-Close4 \cite{Andrews2017}, it is sufficient to review it to exhibit the contributions that will be generalized and incorporated into RIn-Close\_CVC. Respectively, the problems solved by In-Close5 and RIn-Close\_CVC are:

\noindent \textbf{Problem solved by In-Close5}: Mine all maximal CTV biclusters of 1's from a binary data matrix $\mathbf{A}_{n \times m}$, with the enumeration process being (1) efficient, (2) complete, (3) correct, and (4) non-redundant.

\noindent \textbf{Problem solved by RIn-Close\_CVC}: Given a user-defined parameter $\epsilon \ge 0$, mine all maximal CVC biclusters with maximum perturbation $\epsilon$ from a numerical data matrix $\mathbf{A}_{n \times m}$, with the enumeration process being (1) efficient, (2) complete, (3) correct, and (4) non-redundant.


\subsection{In-Close5}
\label{subsec:inclose5}

In discrete mathematics, combinations have a lexicographical order, for instance, \{1, 2, 3\} comes before \{1, 2, 4\}, and also before \{1, 3\} \cite{Andrews2009}. Ganter \cite{Ganter1984} showed how the lexicographical order of formal concepts can be used to avoid the search of repeated results. The CbO algorithm \cite{Kuznetsov1999} and its variants, which is the case of In-Close5 \cite{Andrews2018} and its predecessors \cite{Andrews2009,Andrews2011,Andrews2015,Andrews2017}, implement it in an efficient way, which was named \textit{canonicity test} \cite{Kuznetsov1996}.



Algorithm \ref{alg:inclose5} shows In-Close5 pseudocode. It is invoked with an initial pair $(I,J) = (X,\emptyset)$ (which is called the \textit{supremum} formal concept), an initial attribute index $y = 1$, and two empty sets $P = \emptyset$ and $N = \emptyset$. The difference between In-Close2 and In-Close5 are related to the sets $P$ and $N$, which do not exist in In-Close2 (more details will be given in what follows).

In In-Close algorithms, each formal concept $(I,J)$ is incrementally closed, i.e., its column-set $J$ is incrementally completed with all possible columns for the row-set $I$. So, during the closure of a formal concept, In-Close5 iterates across the attributes (line 1). If the current attribute $j$ is not skipped (line 2), In-Close5 computes the extent of a candidate new formal concept $G$ (line 3). If the size of $G$ is less than the value of the user-defined parameter $minRow$, the descendants of the current formal concept $(I,J)$ can skip the current attribute $j$ because: \textbf{If a parent extent intersected with an attribute-extent results in a set with less than $minRow$ elements, then any subset (child) of the parent extent intersected with the same attribute-extent will also result in a set with less than $minRow$ elements} \cite{Andrews2017}. The set $P$ stores  such attributes (line 15). If the size of $G$ is equal to the size of $I$ (line 5), this means that the row-set $G$ is equal to the current row-set $I$, then the attribute $j$ is added to the current column-set $J$ (line 6). Otherwise, In-Close5 tests if $G$ is canonical (line 8). If yes, the current formal concept $(I, J)$ will give rise to a child formal concept, which is placed in a queue (line 9). After the closure of the current formal concept $(I, J)$, In-Close5 starts closing its descendants (lines 17 to 19). Notice that the descendants of the bicluster $(I, J)$ inherit its intent $J$ (line 18).

The canonicity test is as follows. Letting $J$ be the current column-set, and $j$ be the current attribute, the row-set $G$ of a candidate new formal concept is not canonical if

\vspace*{-0.6cm}

\begin{equation}
\exists k \in Y \setminus J \: \mid \: [k < j] \: \wedge \: [a_{gk} = 1], \forall g \in G.
\label{eq:incloe2_iscan}
\end{equation}

\vspace*{-0.25cm}

\noindent This means that there is an attribute $k < j$, where the submatrix $(G, J \cup \{k\})$ is a CTV bicluster of 1's.

Concerning the update of the set $N$ in the pseudocode. Let $y$ be the starting attribute for the current cycle. Let us say, in a failed canonicity test, that the smallest attribute where the test failed is $k$ (line 11). If $k \ge y$ then an extent, $H$, where $G \subseteq H$, has been discovered in the current cycle at $k$ (and is waiting in the current queue). And there may be other extents discovered after $k$ but before $j$ that are also supersets of $G$ and are also in the queue. Thus, if $k \ge y$, the current attribute, $j$, will be required at the next level to be examined by the children in the queue: $G$ may be canonical with respect to one of the children or $j$ may be an attribute in the intent of a child and thus required to be added. However, if $k < y$, the formal concept with extent $G$ and its children has already been computed and processed. Thus no child in the current queue, or subsequent child, needs to examine $j$. In other words, if $k < y$ then $j$ can be \textbf{inherited as a canonicity test failure} - all subsequent children can skip $j$ in the cycle. The set $N$ stores such attributes (line 13).

\begin{algorithm}[!htb]
\caption{In-Close5$((I,J),y, P, N, \mathbf{A}_{n \times m}, minRow)$}
\label{alg:inclose5}
\begin{algorithmic}[1]
  \small
  \FOR{$j \leftarrow y$ to $m$}
    \IF{$j \notin J$ \AND $j \notin P$ \AND $j \notin N$}
	  \STATE $G \leftarrow I \cap \{j\}^\downarrow$
      \IF{$|G| \geq minRow$}
        \IF{$|G| = |I|$}
          \STATE $J \leftarrow J \cup \{j\}$
        \ELSE
          \IF {$G$ is canonical}
            \STATE PutInQueue($G, j$)
          \ELSE
            \STATE $k \leftarrow $ {\scriptsize lowest column index where the canonicity test fails}
            \IF{$k < y$}
              \STATE $N \leftarrow N \cup \{j\}$
            \ENDIF
          \ENDIF
        \ENDIF
      \ELSE
        \STATE $P \leftarrow P \cup \{j\}$
      \ENDIF
    \ENDIF
  \ENDFOR
  \STATE ProcessConcept($(I,J)$)
  \WHILE{GetFromQueue($G, j$)}
    \STATE $H \leftarrow J \cup \{j\}$
    \STATE In-Close5$((G,H),j+1,P,N, \mathbf{A}_{n \times m}, minRow)$
  \ENDWHILE
\end{algorithmic}
\end{algorithm}

In addition to the minimum number of rows $minRow$, we can easily add a minimum number of columns $minCol$ to In-Close5. While In-Close5 loops through the attributes, a formal concept $(I, J)$ can be discarded if, even adding all remaining attributes to its column-set, it will not meet the minimum number of columns $minCol$ (therefore, its next descendants will not meet the minimum number of columns $minCol$ as well). Although this restriction can be checked only during the closure of a formal concept, it will also prune the search space and save computational resources because ($i$) it stops the construction of a formal concept that will be discarded later, given that it does not meet $minCol$, and ($ii$) it avoids generating descendants that will not meet $minCol$ as well. The same is true for the RIn-Close algorithms \cite{VeronezeEtAl2017}.

\textbf{Worst-case complexity} - The number of recursive calls is equal to the number of formal concepts in the data matrix, $q$. Inside the loop throughout the attributes (line 1 of Algorithm~\ref{alg:inclose5}), the most costly part is the canonicity test. Its worst-case time is $O(nm)$. Thus, the overall worst-case of In-Close5 is $O(qnm^2)$ \cite{Kuznetsov1999}.

In-Close5 has two more structures than In-Close2 (the sets $P$ and $N$). However, it has exactly the same memory usage of In-Close2 because of the strategies adopted by Andrews in his implementation (see the details in his paper \cite{Andrews2018}).

\subsection{RIn-Close\_CVC}
\label{subsec:rinclosecvc}

The generalization of In-Close algorithms to enumerate all maximal perfect CVC biclusters is straightforward and is denoted RIn-Close\_CVCP. We have only one major difference: in In-Close algorithms, each bicluster $(I, J)$ can generate just one descendant per attribute, whereas in RIn-Close\_CVCP, each bicluster $(I, J)$ can generate multiple descendants per attribute \cite{VeronezeEtAl2017}. It happens because In-Close algorithms looks for blocks of 1's, whereas RIn-Close\_CVCP looks for maximal blocks of constant values on columns.

However, the generalization of In-Close algorithms to enumerate all maximal perturbed CVC biclusters is not so simple because not only a bicluster $(I, J)$ is able to generate multiple descendants per attribute, but also \textbf{overlap may occur between their row-sets} \cite{VeronezeEtAl2017}.


In this scenario, since a bicluster $(I, J)$ is able to generate multiple descendants per attribute, \textbf{with a possible superposition among their row-sets} (sharing $minRow$ rows or more in their row-sets), it is necessary to take some actions to avoid the generation of redundant biclusters (either enumerating the same bicluster more than once or mining non-maximal biclusters in their row-sets).

To solve the first problem, i.e., enumerating the same bicluster more than once, the first version of RIn-Close\_CVC \cite{VeronezeEtAl2017} tracks the row-sets that have already been generated using efficient symbol table implementations, such as hash tables (HTs) or balanced search trees (BSTs). The symbol table's keys are given by the row-sets, in such a way that the rows in a row-set are in their ascending (or descending) order. \textbf{Notice that two distinct maximal CVC biclusters must have two distinct row-sets} (see Property~\ref{prop:maxcvc}). Thus, RIn-Close\_CVC \cite{VeronezeEtAl2017} does not mine the same bicluster more than once.

\begin{property}
Let $(I_1,J_1)$ and $(I_2,J_2)$ be two distinct maximal CVC biclusters of a data matrix $\mathbf{A}_{n \times m}$, then $I_1 \neq I_2$ .
\label{prop:maxcvc}
\end{property}
\begin{proof}
Suppose, by way of contradiction, that $(I_1,J_1)$ and $(I_1,J_2)$ are two distinct maximal CVC biclusters of the data matrix $\mathbf{A}_{n \times m}$. Then there is not a CVC bicluster $(I_3, J_3)$ where $I_1 \times J_1 \subset I_3 \times J_3$ or $I_1 \times J_2 \subset I_3 \times J_3$. However, $(I_1,J_1 \cup J_2)$ is a CVC bicluster, which contradicts our assumption.
\end{proof}

To solve the second problem, i.e., mining non-maximal biclusters in their row-sets, RIn-Close\_CVC \cite{VeronezeEtAl2017} verifies if the bicluster is row-maximal. Considering the Definition~\ref{def:maximal}, we see that it can be done by testing all rows in $X \setminus G$, where $G$ is the row-set of the candidate new bicluster. However, it is possible to test a much smaller number of rows. 
We call $\Gamma$ the set of rows that must be checked to verify the row-maximality of the descendants of a bicluster. We will explain how to compute $\Gamma$ in what follows.

Algorithm~\ref{alg:rinclosecvc} presents the pseudocode of RIn-Close\_CVC \cite{VeronezeEtAl2017}. It is invoked with an initial pair $(I, J) = (X, \emptyset)$, an initial attribute $y = 1$, and an empty set $\Gamma = \emptyset$. The symbol table that tracks the biclusters that have already been mined is called $ST$ in the pseudocode.

\begin{algorithm}[!htb]
\caption{RIn-Close\_CVC$((I,J),y,\Gamma, \mathbf{A}_{n \times m}, minRow, \epsilon)$}
\label{alg:rinclosecvc}
\begin{algorithmic}[1]
  \small
  \FOR{$j \leftarrow y$ to $m$}
	  \IF{$j \notin J$}
		  \IF{$\max_{i \in I}(a_{ij}) - \min_{i \in I}(a_{ij}) \leq \epsilon$}
			  \STATE $J \leftarrow J \cup \{j\}$
			\ELSE
			  \STATE Compute the possible new row-sets \COMMENT{Eq.~\ref{eq:rinc_cvc_compExt}}
			  \FOR{each possible new row-set $G$}
					\IF{$|G| \geq minRow$ \AND $G \notin ST$ \AND $G$ is canonical \AND $G$ is row-maximal}
					    \STATE Insert $G$ in the symbol table $ST$
						\STATE $\Omega \leftarrow ComputeRM(G, j, \Gamma, I, \mathbf{A}_{n \times m}, minRow, \epsilon)$ \COMMENT{Algorithm~\ref{alg:ComputeRM}}
						\STATE PutInQueue($G, j, \Omega$)
					\ENDIF					
				\ENDFOR
			\ENDIF
		\ENDIF
	\ENDFOR
    \STATE ProcessBicluster($(I,J)$)
	\WHILE{GetFromQueue($G, j, \Omega$)}
	  \STATE $H \leftarrow J \cup \{j\}$
		\STATE RIn-Close\_CVC($(G,H),j+1, \Omega, \mathbf{A}_{n \times m}, minRow, \epsilon$)
	\ENDWHILE
\end{algorithmic}
\end{algorithm}

Although RIn-Close\_CVC is more elaborate than In-Close algorithms, the pseudocode has the same structure. So, each bicluster $(I,J)$ is incrementally closed, i.e., its column-set $J$ is completed with all possible columns for the row-set $I$. If the attribute $j$ is not an inherited attribute and it cannot be added to the column-set $J$, the possible new row-sets are computed (line 6). Given that $I$ is the current row-set, and $j$ is the current attribute, the possible new row-sets are given by

\vspace*{-0.6cm}

\begin{equation}
 \{G \: | \: [G \subseteq I] \; \wedge \; [\max_{i \in G}(\{a_{ij}\}) - \min_{i \in G}(\{a_{ij}\}) \leq \epsilon] \; \wedge \; [G \; \mathrm{is \; maximal}]\}.
 \label{eq:rinc_cvc_compExt}
\end{equation}

\vspace*{-0.25cm}

\noindent It is easily achieved by \textbf{sorting} the values of the data matrix $\mathbf{A}$ in rows $I$ and column $j$. After ordering, just scroll through the vector looking for the sets of values that meet the user-defined maximum perturbation $\epsilon$.


Letting $J$ be the current column-set, and $j$ be the current attribute, the row-set $G$ of a possible new bicluster is not canonical if

\vspace*{-0.6cm}

\begin{equation}
    \exists k \in Y \setminus J \: | \: [k < j] \: \wedge \: [\max_{i \in G}(a_{ik}) - \min_{i \in G}(a_{ik}) \leq \epsilon],
	\label{eq:rinc_cvc_iscan}
\end{equation}

\vspace*{-0.25cm}

\noindent i.e., if there is an attribute $k < j$ that we can add to the bicluster $(G, J)$ and it remains a CVC bicluster. Also, the candidate new bicluster with row-set $G$ is not row-maximal if there is an object $g \in \Gamma$ that we can add to the bicluster $(G,J \cup \{j\})$ and it remains a CVC bicluster, i.e.,

\vspace*{-0.6cm}

\begin{equation}
\exists g \in \Gamma \: | \: \max_{i \in \{G \cup \{g\}\}}(a_{ik}) - \min_{i \in \{G \cup \{g\}\}}(a_{ik}) \leq \epsilon, \forall k \in J \cup \{j\}. 
\label{eq:cvc_ismaximal}
\end{equation}
 
 \vspace*{-0.25cm}

\begin{algorithm}[!htb]
\caption{ComputeRM$(G, j, \Gamma, I, \mathbf{A}_{n \times m}, minRow, \epsilon)$}
\label{alg:ComputeRM}
\begin{algorithmic}[1]
  \small
  \ENSURE new set of rows to check the row-maximality $\Gamma_{new}$
  \STATE $p1 \leftarrow minRow$-$th$ smaller element of $\{a_{ij}\}_{i \in G}$ \COMMENT{pivot value 1}
  \STATE $p2 \leftarrow minRow$-$th$ larger element of $\{a_{ij}\}_{i \in G}$ \COMMENT{pivot value 2}
  \STATE $\Gamma_{new} \leftarrow \Gamma \cup \{i \in I \setminus G \mid [p1 - a_{ij} \leq \epsilon] \;  \vee \; [a _{ij} -p2 \leq \epsilon] \}$
\end{algorithmic}
\end{algorithm}

Now, let us explain the function $ComputeRM$ of Algorithm~\ref{alg:ComputeRM}. The pivot elements are the $minRow$-$th$ smaller and larger elements of $\{a_{ij}\}_{i \in G}$. To exemplify, let us suppose that $\{a_{ij}\}_{i \in G} = \{3, 3, 4, 4.5, 5, 6\}$, $minRow = 2$, and $\epsilon = 3$.  So, $p1 = 3$ and $p2 = 5$. Considering the current attribute $j$, rows $i \in I \setminus G$ with values greater than or equal to 0 ($p1 - \epsilon$) or less than or equal to 8 ($p2 + \epsilon$) must comprise $\Gamma_{new}$. In addition, 
the set of rows to check the row-maximality from its parent must be inherited. 
See our original paper \cite{VeronezeEtAl2017} for an illustration that exemplifies it intuitively.

\textbf{Worst-case complexity} - Since the pseudocode of RIn-Close\_CVC has the same structure of In-Close algorithms, let us examine the differences. RIn-Close\_CVC loops through the possible new row-sets, having worst-case complexity $O(n)$. The worst-case time of checking if a candidate new bicluster is row-maximal is the same as the canonicity test: $O(nm)$. The worst-case time to insert and search in a BST is $O(\log q)$, where $q$ is its number of elements. The worst-case time to insert and search in a HT is $O(1)$ and $O(q)$, respectively. However, under reasonable assumptions, the average time to search in a HT is $O(1)$ \cite{CormenEtAl2009}. We sort the row-set to insert it into the symbol table, and its worst-case time is $O(n \log n)$. The worst-case time of the function \textit{ComputeRM} is $O(n)$.  Thus, the overall worst-case time of RIn-Close\_CVC is $O(qnm(n \log n + nm + x))$,  where $x$ is the worst-case time of searching in the symbol table. Remark that usually $\log n \ll m$.

\begin{proposition}
RIn-Close\_CVC is an (1) efficient, (2) complete, (3) correct, and (4) non-redundant algorithm for mining all maximal CVC biclusters, with maximum perturbation $\epsilon$, from a numerical data matrix $\mathbf{A}_{n \times m}$.
\end{proposition}
\begin{proof}
RIn-Close\_CVC is a generalization of In-Close2, so we need only to show that the generalization steps keep these 4 properties:

\noindent (1) Efficiency: RIn-Close\_CVC has also at most polynomial time in the input size. If it is using a BST, it has worst-case time complexity quasi-linear, more specifically, linearithmic in the number of biclusters. If it is using a HT, its worst-case time is quadratic in the number of biclusters, however, under reasonable assumptions, it will have linear time in the number of biclusters.

\noindent (2) Completeness: RIn-Close\_CVC has the same search engine as In-Close2, being the biclusters incrementally closed. The main difference is that each bicluster $(I, J)$ can generate more than one descendant per attribute, possibly with overlaps between their row-sets. But this difference does not affect the search engine. Once the row-sets of the candidate new biclusters are computed according to Eq.~\ref{eq:rinc_cvc_compExt}, the RIn-Close\_CVC's procedure for each candidate new bicluster is the same as the In-Close2's procedure. Also, the row-maximal checking will only discard non-maximal biclusters in their row-sets. Lastly, since each maximal CVC bicluster has a unique row-set, the usage of the symbol table $ST$ will only avoid the mining of the same maximal bicluster more than once. So, these features of RIn-Close\_CVC do not interfere in the completeness of the algorithm.

\noindent (3) Correctness: From the search engine of In-Close2 (inherited by RIn-Close\_CVC), we can observe that a descendant bicluster of a bicluster $(I,J)$ has row-set $G \subset I$, and column-set $H \supset J$ (notice that the descendants inherit the column-set $J$ of their parent). For all $h \in H$, we must check if

\vspace*{-0.6cm}

\begin{equation*}
\max_{i \in G}(\{a_{ih}\}) - \min_{i \in G}(\{a_{ih}\}) \le \epsilon.
\end{equation*}

\vspace*{-0.25cm}

\noindent Note that we can divide the column-set $H$ of the new bicluster in three disjoint subsets: $\{j\}$, $J$, and $H \setminus \{J \cup \{j\}\}$, where $j$ is the attribute where the bicluster was created, $J$ are the set of inherited columns, and $H \setminus \{J \cup \{j\}\}$ are the set of columns that the bicluster gained during its closure. For the subsets $\{j\}$ and $H \setminus \{J \cup \{j\}\}$, this check is ok by construction (see Eq.~\ref{eq:rinc_cvc_compExt} and line 3 of Algorithm~\ref{alg:rinclosecvc}, respectively). Supposing that the parent bicluster $(I,J)$ is correct, so the bicluster $(G,J)$ is also correct because $G \subset I$. We can affirm that the parent bicluster $(I,J)$ is correct by induction, since ($i$) a bicluster having the column in which it was generated and the columns it gained during its closure is correct by construction, and ($ii$) the supremum bicluster (ancestor of all biclusters) starts with no columns to pass to its descendants.

\noindent (4) Non-Redundancy:
As in In-Close2, the column-set of a bicluster is maximal after its closure. However, we must check if a candidate new bicluster with row-set $G$ is row-maximal. We could do this by testing all rows in $X \setminus G$, but it is sufficient to test the rows in the set $\Gamma$. Why? The rows in $\Gamma$ are computed taking into account that each bicluster must have at least $minRow$ rows, and all its columns must have at most perturbation $\epsilon$. Thus, the attribute $j$ of the new bicluster, where it was generated, is used as reference to compute a set of rows that must belong to the set $\Gamma$: based on the $minRow$-$th$ smaller and $minRow$-$th$ larger values of the column $j$ of the new bicluster, rows that could generated a CVC bicluster are added to $\Gamma$. Furthermore, the set $\Gamma$ of a new bicluster inherits the set of rows to check the row-maximality from all its ancestors. With this inheritance, all rows that could possibly generate a CVC bicluster are kept tracked.

In addition to the inherited canonicity test of In-Close2, RIn-Close\_CVC also uses a symbol table for the prevention of mining the same (maximal) bicluster more than once. The keys of the symbol table are the ordered row-sets of the mined biclusters. Thus, as each maximal CVC bicluster has a unique row-set, RIn-Close\_CVC will not enumerate the same bicluster more than once.
\end{proof}

\section{RIn-Close\_CVC3}
\label{sec:rinclosecvc3}

We saw in Subsection~\ref{subsec:rinclosecvc} that during the closure of a CVC bicluster $(I,J)$, it can be generated multiple descendants per attribute with overlap between their row-sets. Therefore, without some extra care, an enumerative algorithm based on the search engine of In-Close2 could return redundant biclusters (either enumerating the same bicluster more than once or mining non-maximal biclusters in their row-sets), thus bringing forth the \textit{overlapping row-sets problem}.

RIn-Close\_CVC3 deals with the \textit{overlapping row-sets problem} in a more efficient way than RIn-Close\_CVC \cite{VeronezeEtAl2017}, thereby gathering relevant improvements to the task of biclusters' enumeration. Besides that, we prepare this new version to handle datasets with missing values, attributes with distinct distributions, and mixed data types, fully exploring the definition of a CVC bicluster given by Definition~\ref{def:cvcbic2}. Lastly, we generalize the contributions proposed in \cite{Andrews2017,Andrews2018} to incorporate them into our new version. Notice that RIn-Close\_CVC \cite{VeronezeEtAl2017} is conceptually based on In-Close2 \cite{Andrews2011}.

Algorithm~\ref{alg:rinclosecvc3} shows the pseudocode of RIn-Close\_CVC3. Unlike RIn-Close\_CVC, it does not have a symbol table $ST$ and neither the checking of row-maximality. Instead, it has a checking that we call \textit{row-canonical} to deal with the \textit{overlapping row-sets problem}. The expression $a_{ij} \neq mv$ means: the element $a_{ij}$ is not a missing value. It is invoked with an initial pair $(I, J) = (X, \emptyset)$, an initial attribute $y = 1$, and two empty sets $\Gamma = \emptyset$ and $PN = \emptyset$.

In this new version, the possible new row-sets are given by

\vspace*{-1.2cm}

\begin{equation}
\{G \mid [G \subseteq I] \; \wedge \; [\max_{i \in G}(\{a_{ij}\}) - \min_{i \in G}(\{a_{ij}\}) \leq \epsilon_j] \; \wedge \; [a_{ij} \neq mv, \forall i \in G] \; \wedge \; [G \; \mathrm{is \; maximal}]\},
\label{eq:rinc_cvc3_compExt}
\end{equation}

\vspace*{-0.25cm}

\noindent given that $I$ is the current row-set, and $j$ is the current attribute. As we already know, it is easily achieved by ordering the values of the data matrix $\mathbf{A}_{n \times m}$ in rows $I$ and column $j$. The user should use a large number to represent the missing values ($mv$), so the missing values will be at the bottom of the ordered list, and this whole portion of the list can be ignored.

The canonicity test is update as follows. Letting $J$ be the current column-set, and $j$ be the current attribute, a possible new row-set $G$ of a bicluster is not canonical if

\vspace*{-0.6cm}

\begin{equation}
\exists k \in Y \setminus J \mid [k < j] \: \wedge \: [\max_{i \in G}(a_{ik}) - \min_{i \in G}(a_{ik}) \leq \epsilon_k] \: \wedge \: [a_{ik} \neq mv, \forall i \in G].
\label{eq:rinc_cvc3_iscan}
\end{equation} 

\vspace*{-0.25cm}


\begin{algorithm}[!htb]
\caption{RIn-Close\_CVC3$((I,J),y,\Gamma,PN,\mathbf{A}_{n \times m}, minRow, \bm{\epsilon})$}
\label{alg:rinclosecvc3}
\begin{algorithmic}[1]
\small
\FOR{$j \leftarrow y$ to $m$}
 \IF{$j \notin J$ \AND $j \notin PN$}
   \IF{$\max_{i \in I}(a_{ij}) - \min_{i \in I}(a_{ij}) \leq \epsilon_j$ \AND $a_{ij} \neq mv, \forall i \in I$}
    \STATE $J \leftarrow J \cup \{j\}$
   \ELSE
    \STATE $a1 \leftarrow a2 \leftarrow true$, $a3 \leftarrow false$
    \STATE Compute the possible new row-sets \COMMENT{Eq.~\ref{eq:rinc_cvc3_compExt}}
	\FOR{each possible new row-set $G$}
	 \IF{$|G| \geq minRow$}
      \STATE $a1 \leftarrow false$
      \IF {$G$ is canonical}
       \STATE $a2 \leftarrow false$
       \IF {$G$ is row-canonical}
  	    \STATE $\Omega \leftarrow ComputeRM(G, j, \Gamma, I, \mathbf{A}_{n \times m}, minRow, \epsilon_j)$ \COMMENT{Algorithm~\ref{alg:ComputeRM}}
	    \STATE PutInQueue($G, j, \Omega$)
       \ENDIF
      \ELSE
       \STATE $k \leftarrow $ {\scriptsize lowest column index where the canonicity test fails}
       \IF{$k \ge y$}
        \STATE $a2 \leftarrow false$
       \ELSE
        \STATE $a3 \leftarrow true$
       \ENDIF
      \ENDIF
	 \ENDIF					
    \ENDFOR
    \IF{$a1 = true$ \OR ($a2 = true$ \AND $a3 = true$)}
     \STATE $PN \leftarrow PN \cup \{j\}$
    \ENDIF
   \ENDIF
  \ENDIF
 \ENDFOR
 \STATE ProcessBicluster($(I,J)$)
 \WHILE{GetFromQueue($G, j, \Omega$)}
  \STATE $H \leftarrow J \cup \{j\}$
  \STATE RIn-Close\_CVC3($(G,H),j+1, \Omega, PN,\mathbf{A}_{n \times m}, minRow, \bm{\epsilon}$)
 \ENDWHILE
\end{algorithmic}
\end{algorithm}

The new row-canonicity works as follows. Letting $J$ be the current column-set, $j$ the current attribute, $J^{<j}$ the set of all attributes of $J$ up to $j$, $H = J^{<j} \cup \{j\}$, and $\Gamma$ the set of rows that must be checked to verify the row-canonicity, the candidate new bicluster with row-set $G$ is not row-canonical if (1) there is an object $g \in \Gamma$ that we can add to the bicluster $(G,H)$ so that it remains a CVC bicluster, i.e.,

\vspace*{-0.6cm}

\begin{equation}
  \exists g \in \Gamma \: | \: [\max_{i \in \{G \cup \{g\}\}}(a_{ik}) - \min_{i \in \{G \cup \{g\}\}}(a_{ik}) \leq \epsilon_k] \: \wedge \: [a_{gk} \neq mv], \forall k \in H,
	\label{eq:cvc_row-iscan1}
\end{equation}

\vspace*{-0.25cm}

\noindent or (2) there is another bicluster with \textbf{lower lexicographic order} in the rows that could be the parent of the candidate new bicluster, i.e.,

\vspace*{-0.6cm}

\begin{equation}
\exists g \in \Gamma \: | \: [\max_{i \in \{I^{<g} \cup \{g\} \cup G\}}(a_{ik}) - \min_{i \in \{I^{<g} \cup \{g\} \cup G\}}(a_{ik}) \leq \epsilon_k] \: \wedge \: [a_{gk} \neq mv], \forall k \in J^{<j},
\label{eq:cvc_row-iscan2}
\end{equation}


\noindent where $I^{<g}$ is the set of all objects of $I$ up to $g$ ($g \notin I$ by definition).

The idea behind the row-canonical function is as follows. If a bicluster can be created from more than one bicluster on different columns, it must be created in the most posterior column. If a bicluster can be created from more than one bicluster in the same column $j$, it must be created by the bicluster with the lowest lexicographic order in its row-set. Besides, the first part of the function row-canonical also ensures that a new bicluster is row-maximal.

The set of rows that must be checked to verify the row-canonicity is also computed by Algorithm~\ref{alg:ComputeRM}. However, line 3 of Algorithm~\ref{alg:ComputeRM} must be updated to

\vspace*{-0.6cm}

\begin{equation*}
\Gamma_{new} \leftarrow \Gamma \cup \{i \in I \setminus G \mid [[p1 - a_{ij} \leq \epsilon_j] \;  \vee \; [a _{ij} -p2 \leq \epsilon_j]] \; \wedge \; [a_{ij} \neq mv]\},
\end{equation*}

\vspace*{-0.25cm}

\noindent in order to handle missing values and Definition~\ref{def:cvcbic2}.

Fig.~\ref{fig:rowCanon} exhibits four synthetic datasets. We are going to use them to illustrate how the row-canonical test works. In what follows, let us consider the following parameters: $minRow = 2$ and $\epsilon = 1$. 

Concerning the dataset of Fig.~\ref{fig:rowCanon}(a). There is a bicluster with row-set $I_a = \{1, 2, 3, 4, 5\}$, whose set of rows to check the row-maximality is $\Gamma_a = \{6, 7\}$. During its closure, this bicluster generates a possible descendant bicluster in attribute $3$ with row-set $\{3, 4, 5\}$, but the row-canonical test fails when testing row $6 \in \Gamma_a$ since it is a non-row-maximal bicluster. During the closure of the bicluster with row-set $I_b = \{3, 4, 5, 6, 7\}$, whose set of rows to check the row-maximality is $\Gamma_b = \{1, 2\}$, it generates a descendant bicluster in attribute $3$ with row-set $I_c = \{3, 4, 5, 6\}$.

Considering the dataset of Fig.~\ref{fig:rowCanon}(b). There is a bicluster with row-set $I_b = \{3, 4, 5, 6, 7\}$, whose set of rows to check the row-maximality is $\Gamma_b = \{1, 2\}$. During its closure, it generates a possible descendant bicluster in attribute $3$ with row-set $I_c = \{3, 4, 5\}$, but the row-canonical test fails when testing row $1 \in \Gamma_b$ (or $2 \in \Gamma_b$) since there is another bicluster with lower lexicographic order in its row-set that can generate it. So, during the closure of the bicluster with row-set $I_a = \{1, 2, 3, 4, 5\}$, whose set of rows to check the row-maximality is $\Gamma_a = \{6, 7\}$, it creates the bicluster with row-set $I_c$ in attribute $3$. This bicluster inherits the set of rows to check the row-maximality from its parent, so $\Gamma_c \subseteq \Gamma_a$. During its closure, this bicluster generates a possible descendant bicluster in attribute $4$ with row-set $\{4, 5\}$, but the row-canonical test fails when testing row $7 \in \Gamma_c$ since it is a non-row-maximal bicluster. During the closure of the bicluster with row-set $I_b$, it generates a descendant bicluster in attribute $3$ with row-set $I_d = \{4, 5, 7\}$. This example illustrates why the set of rows to check the row-maximality of a bicluster must be inherited by its descendants.

In the case of the dataset of Fig.~\ref{fig:rowCanon}(c), there is a bicluster with row-set $I_b = \{3, 4, 5, 6, 7\}$, whose set of rows to check the row-maximality is $\Gamma_b = \{1, 2\}$. During its closure, it generates a possible descendant bicluster in attribute $3$ with row-set $I_c = \{3, 4\}$, but the row-canonical test fails when testing row $1 \in \Gamma_b$ (or $2 \in \Gamma_b$) since there is another bicluster with lower lexicographic order in its row-set that can generate it. So, during the closure of the bicluster with row-set $I_a = \{1, 2, 3, 4, 5\}$, it creates the bicluster with row-set $I_c$ in attribute $3$.

Concerning the dataset of Fig.~\ref{fig:rowCanon}(d). There is a bicluster with row-set $I_a = \{1, 2, 3, 4, 5\}$, whose set of rows to check the row-maximality is $\Gamma_a = \{6, 7\}$. During its closure, this bicluster generates a possible descendant bicluster in attribute $3$ with row-set $I_d = \{4, 5\}$, but the row-canonical test fails when testing row $6 \in \Gamma_a$ (or $7 \in \Gamma_a$) since there is another bicluster that can generate it in a posterior column. So, during the closure of the bicluster with row-set $I_b = \{3, 4, 5, 6, 7\}$, it creates the bicluster with row-set $I_c = \{4, 5, 6, 7\}$ in attribute $3$, and the bicluster with row-set $I_c$ generates the bicluster with row-set $I_d$ in attribute $4$.

\begin{figure}[!htb]
\centering
\includegraphics[trim=2.5cm 13cm 7cm 3.5cm, clip, scale=.5]{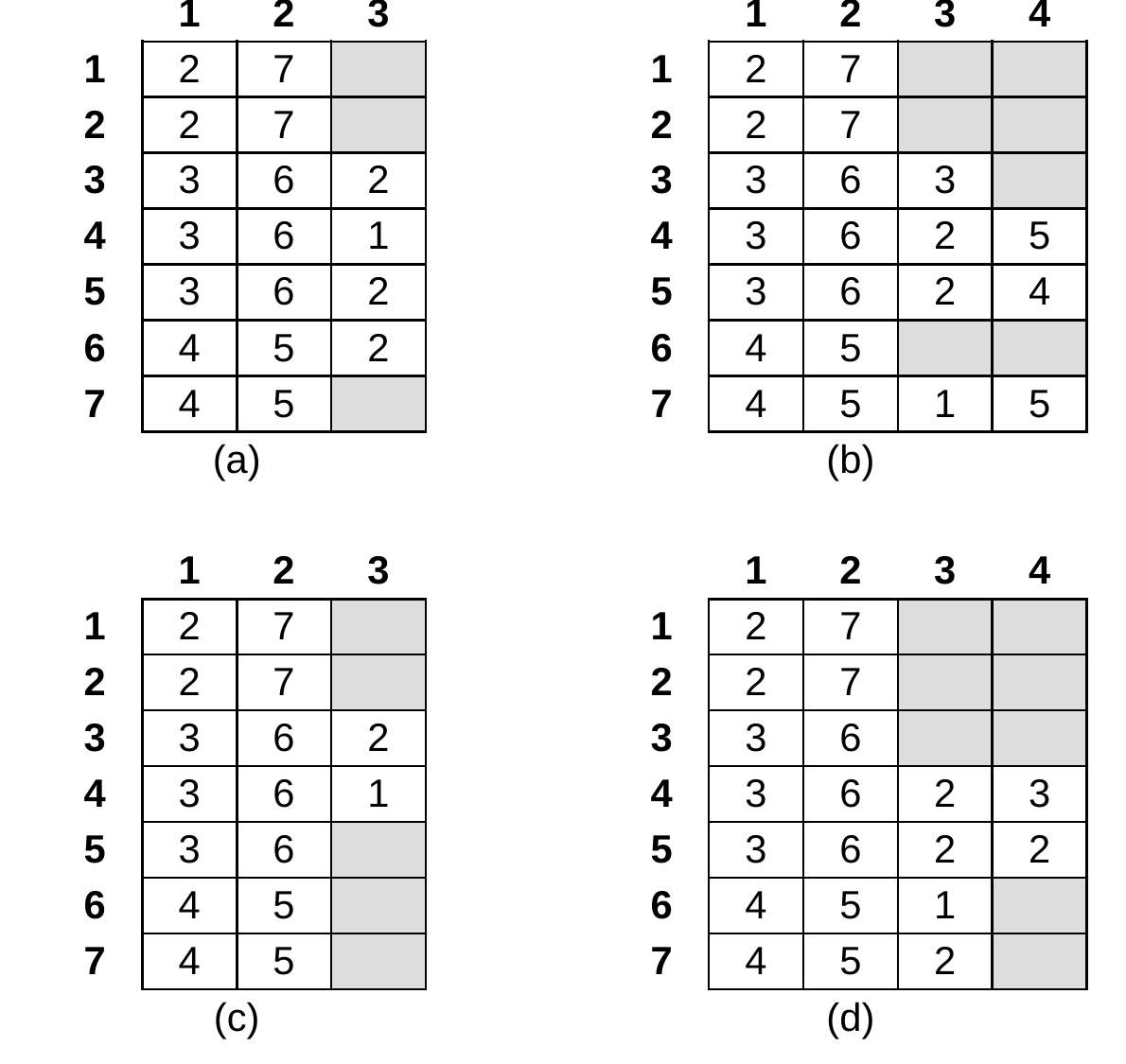}
\caption{Four synthetic datasets. Hatched cells represent missing values.}
\label{fig:rowCanon}
\end{figure}

So far we covered the main differences between the RIn-Close\_CVC3 and RIn-Close\_CVC algorithms. But we still have to explain how we have incorporated the contributions proposed in \cite{Andrews2017,Andrews2018} into RIn-Close\_CVC3, which was done by means of the set $PN$ in the pseudocode. We presented the In-Close5 pseudocode with two distinct sets, $P$ and $N$, for didactic purposes, but as explained by Andrews \cite{Andrews2018}, the information of the sets $P$ and $N$ can be stored in the same set, thus saving computational resources.

Following Andrew's reasoning \cite{Andrews2018}, we have that if the current bicluster $(I,J)$ is unable to produce, in the current attribute $j$, a candidate row-set $G$ that meets the size restriction $minRow$, so will its descendants. Also, if the canonicity test fails in some attribute $k<y$ (where $y$ is the starting attribute for the current cycle) for all candidate row-set $G$ that meets the size restriction $minRow$, 
then $j$ can be inherited as a canonicity test failure. In short, the set $PN$ will be updated (line 23 of Algorithm~\ref{alg:rinclosecvc3}) if no candidate new row-set $G$ (1) meets the size restriction $minRow$, or (2) is canonical and none of the canonicity tests fails in an attribute $k \ge y$.

Our implementation of RIn-Close\_CVC3 has a structure to store the set $PN$. We could use the same strategy adopted by Andrews \cite{Andrews2018} to avoid this new structure, but it would require several modifications to our current implementation. So, as it has an insignificant impact on the memory usage (see our experimental results in Section~\ref{sec:expresults}), we chose to use this extra structure.

\textbf{Worst-case complexity} - The symbol table and the row-maximal test used in RIn-Close\_CVC are replace by the row-canonical test in RIn-Close\_CVC3, having worst-case time $O(n^2m)$. Thus, the overall worst-case time of RIn-Close\_CVC3 is $O(qn^3m^2)$.

\begin{proposition}
RIn-Close\_CVC3 is an (1) efficient, (2) complete, (3) correct, and (4) non-redundant algorithm for mining all maximal CVC biclusters, with maximum perturbation $\epsilon_j$ for each column $j$, from a numerical data matrix $\mathbf{A}_{n \times m}$.
\end{proposition}
\begin{proof}

We only need to show that the new features of RIn-Close\_CVC3 (when compared to RIn-Close\_CVC) keeps these 4 properties.

\noindent (1) Efficiency: RIn-Close\_CVC3 has time complexity linear in the number of biclusters and polynomial in the input size.

\noindent (2) Completeness: The row-canonicity will not lose a (maximal) bicluster because a new candidate bicluster will only be discarded if (1) there is a bicluster $(G \cup \{g\}, J^{<j} \cup \{j\})$, which may create it in a column greater than $j$ ($g \in \Gamma$), or (2) there is another bicluster with lower lexicographic order in its row-set, which may create it in the column $j$.

\noindent (3) Correctness: Essentially, the difference is in the computation of row-sets of possible new biclusters (see Eq.~\ref{eq:rinc_cvc3_compExt}), which also takes into account missing values and a maximum perturbation $\epsilon_j$ per column $j$. The biclusters do not contain missing values and are correct according to Definition~\ref{def:cvcbic2}.

\noindent (4) Non-Redundancy:
We have already shown that the set $\Gamma$ contains all rows that could possibly generate a CVC bicluster due to the \textit{overlapping row-sets problem}.

The first part of the row-canonicity test (see Eq.~\ref{eq:cvc_row-iscan1}) ensures that a candidate new bicluster that is not row-maximal will be discarded.

In the case of mining the same bicluster more than once, we can have two situations: a bicluster can be created by more than one bicluster (1) on different columns, and/or (2) in a same column $j$. Both situations are covered by the new row-canonicity test. Its first part ensures that a new bicluster will be created in the most posterior column by checking if there exists a CVC bicluster $(G \cup \{g\}, J^{<j} \cup \{j\})$. The second part ensures that a new bicluster will only be created in a column $j$ by the bicluster with the lowest lexicographical order in its row-set. Therefore, the test fails if there is a CVC bicluster with row-set $\{I^{<g} \cup \{g\} \cup G\} < I$ and column-set $J^{<j}$ (remark that $G \subset I$ and if $I^{<g}  = I$, then the bicluster $(I^{<g} \cup \{g\} \cup G, J^{<j})$ is incorrect by definition since $I$ is a maximal row-set).
\end{proof}

\section{Experimental results}
\label{sec:expresults}

We evaluated RIn-Close\_CVC3 on both synthetic and real-world datasets, and we will outline here the advantages of this new version when compared to its previous versions. First of all, we evaluated RIn-Close\_CVC3 on synthetic datasets and tested its performance when varying fundamental characteristics of a dataset, thus providing an overall estimative about its scalability in practice. Secondly, we tested RIn-Close\_CVC3's performance on real-world datasets by analyzing its sensitivity to the maximum perturbation $\epsilon$ allowed in a bicluster, and performed a comparison between the ability to extract information using online and a priori partitioning. The results produced by the online partitioning of RIn-Close\_CVC3 will clearly reveal that a priori partitioning, even followed by enumerative biclustering, involves loss of information. Lastly, we applied RIn-Close\_CVC3 in the analysis of mixed-attribute datasets, aiming to discover supervised descriptive rules. We again performed a comparison between the solutions provided by using a priori partitioning and RIn-Close\_CVC3. The results also show that RIn-Close\_CVC3 is able to mine interesting rules that can appropriately describe the decision variable of a labeled dataset.

The experiments were carried out on a PC Intel(R) Core(TM) i7-4770K CPU @ 3.5 GHz, 32GB of RAM, and running under Ubuntu 14.04.

\subsection{Scalability issues}

This experiment aims to test RIn-Close\_CVC3's performance when varying ($i$) the number $n$ of rows of the dataset, ($ii$) the number $m$ of columns of the dataset, ($iii$) the number of biclusters in the dataset, ($iv$) the bicluster row size, ($v$) the bicluster column size, ($vi$) the overlap among the biclusters, and ($vii$) the percentage of missing values. For this purpose, we created synthetic datasets with controlled number, size, shape and level of noise of the existing biclusters and, then, we tested how RIn-Close\_CVC3 performs when varying each one of the parameters in isolation.

The default parameters used in the synthetic data generator were: $n = 10,000$; $m = 100$; number of biclusters $= 30$; bicluster row size $= 200$; bicluster column size $= 16$; overlap $= 0.2$; percentage of missing value $= 0$; and Gaussian noise with $\mu = 0$ and $\sigma = 0.05$. The synthetic data generator creates the biclusters and assigns random values to the other regions of the dataset. Then, it adds Gaussian noise and shuffles the rows and columns of the dataset. Therefore, arbitrarily positioned overlapping biclusters are produced, so that the resulting biclusters are usually non-contiguous. The amount of noise was chosen in such a way that the original biclusters were preserved. 

For each configuration, we created 50 different synthetic datasets to compute the median runtimes and memory usage. We chose the median, not the mean, because the median is less sensitive to outliers. For the same configuration, the results may vary because the positioning of the biclusters is random. From FCA, we know that the placement of the biclusters in a dataset makes all difference in the runtime because less failures in the canonicity tests imply faster execution time \cite{CarpinetoEtAl2004}.

Remark that all versions of RIn-Close\_CVC produce exactly the same biclustering solution, finding all the planted biclusters without losing any row or column. Figure~\ref{fig:expSynDataRT} shows the runtime and the memory usage for the different configurations. We included RIn-Close\_CVCP as a baseline for comparison. Since RIn-Close\_CVCP looks for perfect biclusters, it was applied to datasets without the Gaussian noise. Figure~\ref{fig:expSynDataRT} shows the memory usage only when varying the number $n$ of rows of the dataset since the results are essentially the same for all other variables.

The memory usage of RIn-Close\_CVC3 was equivalent to the memory usage of RIn-Close\_CVCP and RIn-Close\_CVC2, being much better than the one of RIn-Close\_CVC. Therefore, the alternative presented here to the symbol table used in RIn-Close\_CVC promotes a breakthrough in terms of memory usage. Besides, the extra vector used in our implementation of RIn-Close\_CVC3 (to store the set $PN$ of Algorithm~\ref{alg:rinclosecvc3}), when compared to RIn-Close\_CVC2, had no impact in the memory usage. RIn-Close\_CVC3 also brings an overall significant runtime gain, even having a higher worst-case time-complexity than RIn-Close\_CVC.  In terms of runtime, RIn-Close\_CVC3 strongly outperforms its previous versions, and even RIn-Close\_CVCP. It indicates that an expressive gain was achieved with the incorporation of the contributions proposed in \cite{Andrews2017, Andrews2018} into the RIn-Close algorithms. Notice also that the runtime of RIn-Close\_CVC3 was even close to the runtime of RIn-Close\_CVCP*, which is RIn-Close\_CVCP updated with these contributions.

\begin{figure}[!htb]
\centering
\subfigure[]{
  \includegraphics[trim=0.2cm 0.1cm 0.6cm 0.4cm, clip, scale=0.3]{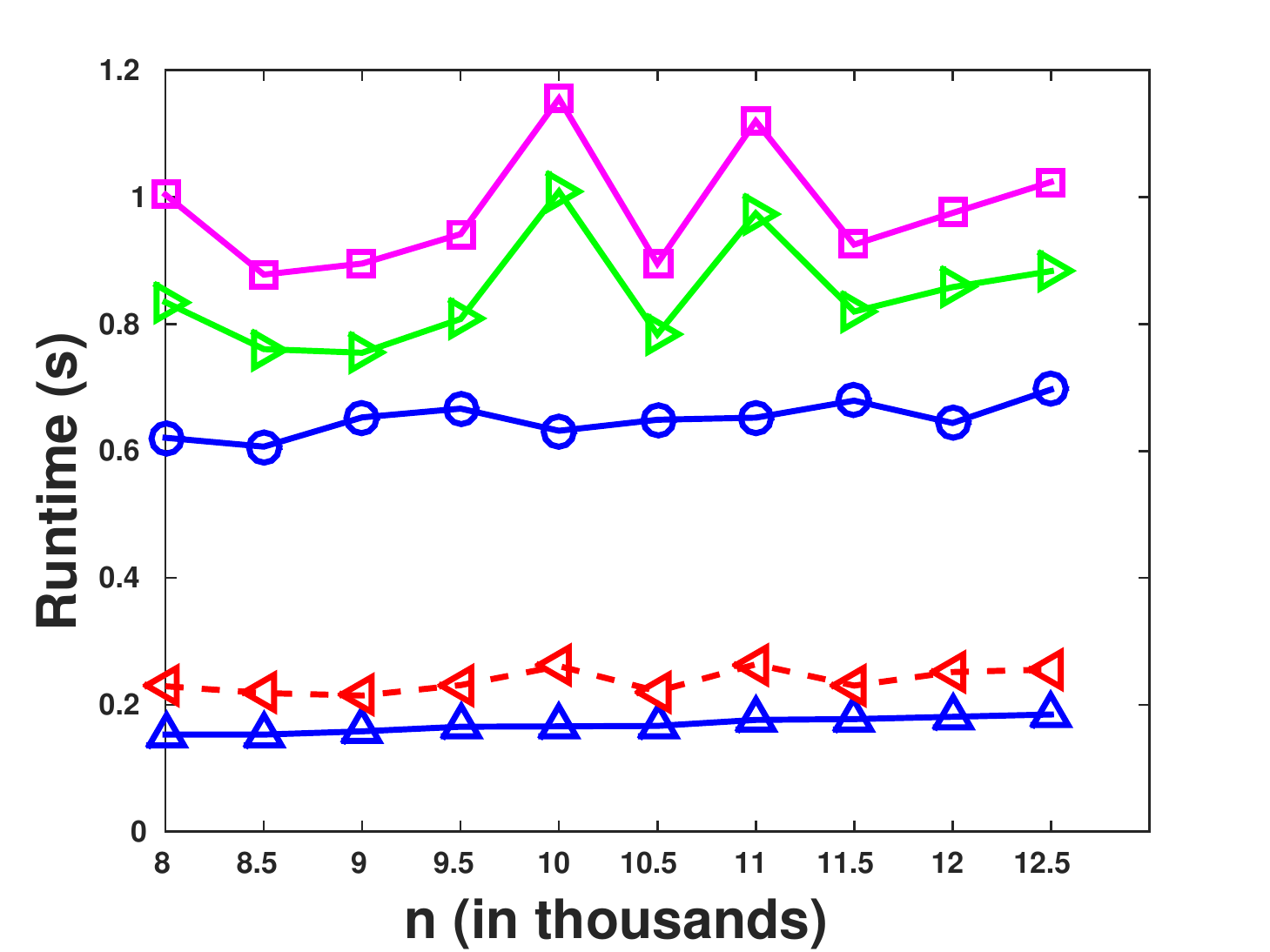}
}
\subfigure[]{
  \includegraphics[trim=0.2cm 0.1cm 0.6cm 0.4cm, clip, scale=0.3]{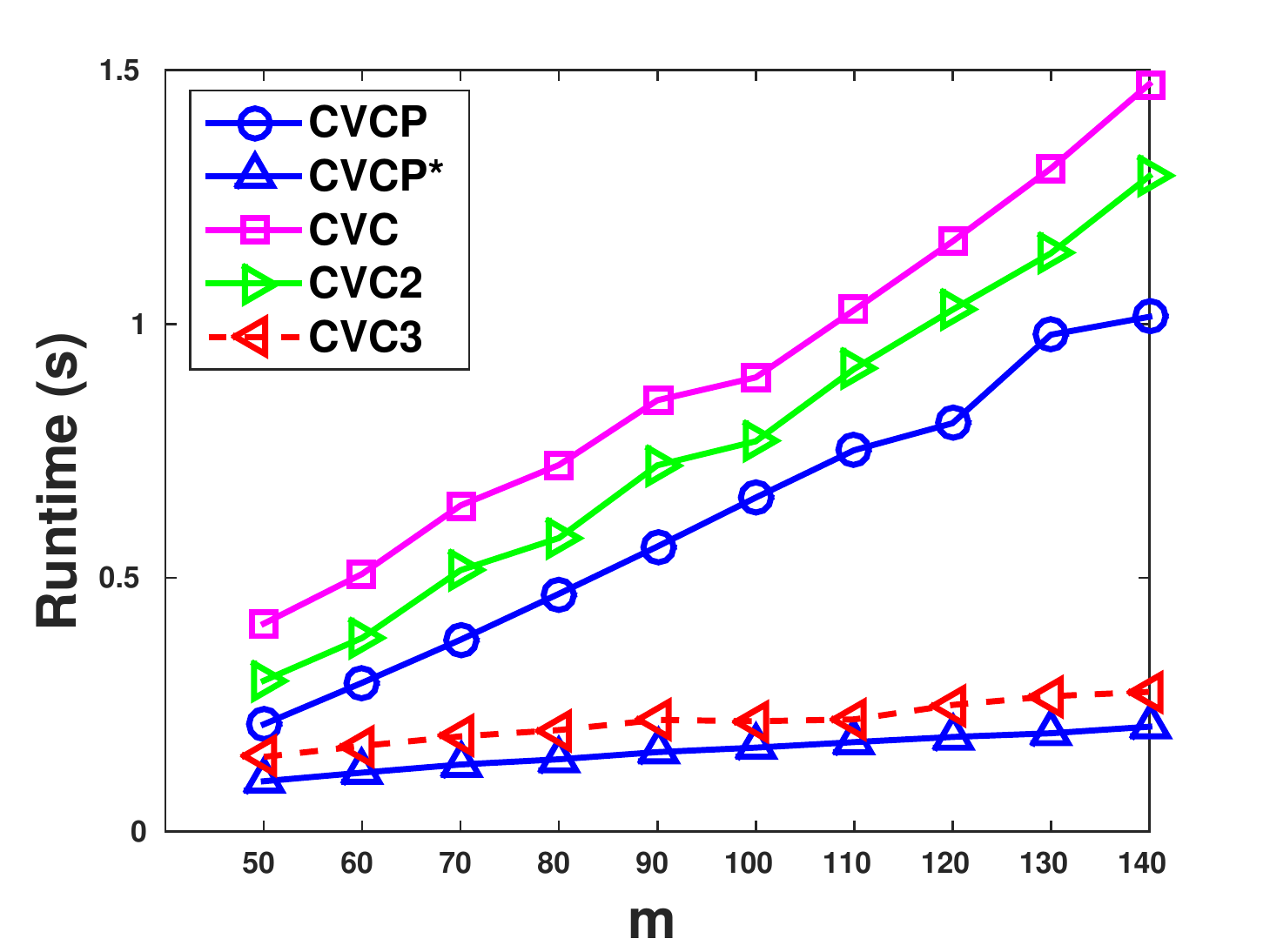}
}
\subfigure[]{
  \includegraphics[trim=0.2cm 0.1cm 0.6cm 0.4cm, clip, scale=0.3]{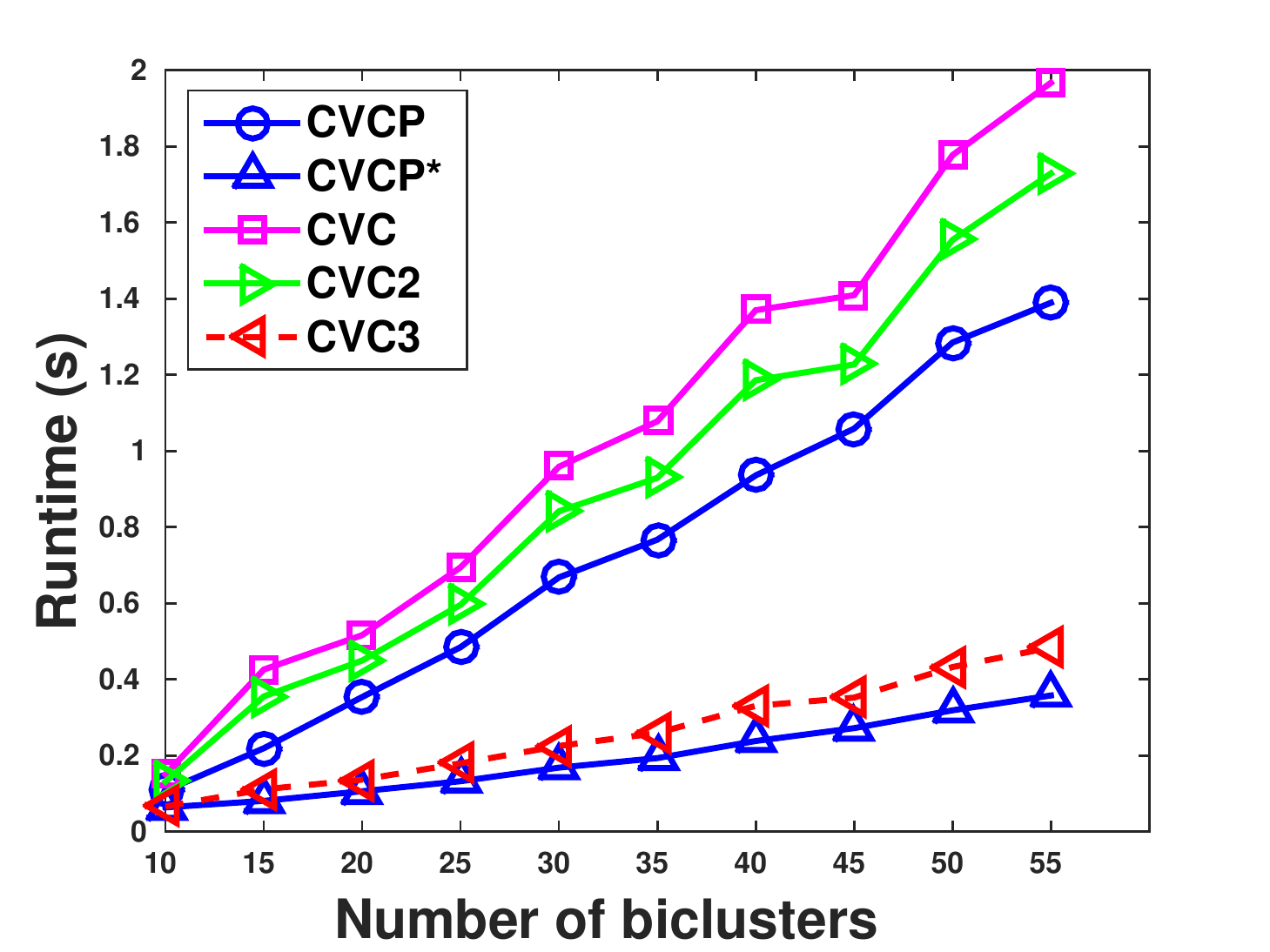}
}
\subfigure[]{
  \includegraphics[trim=0.2cm 0.1cm 0.6cm 0.4cm, clip, scale=0.3]{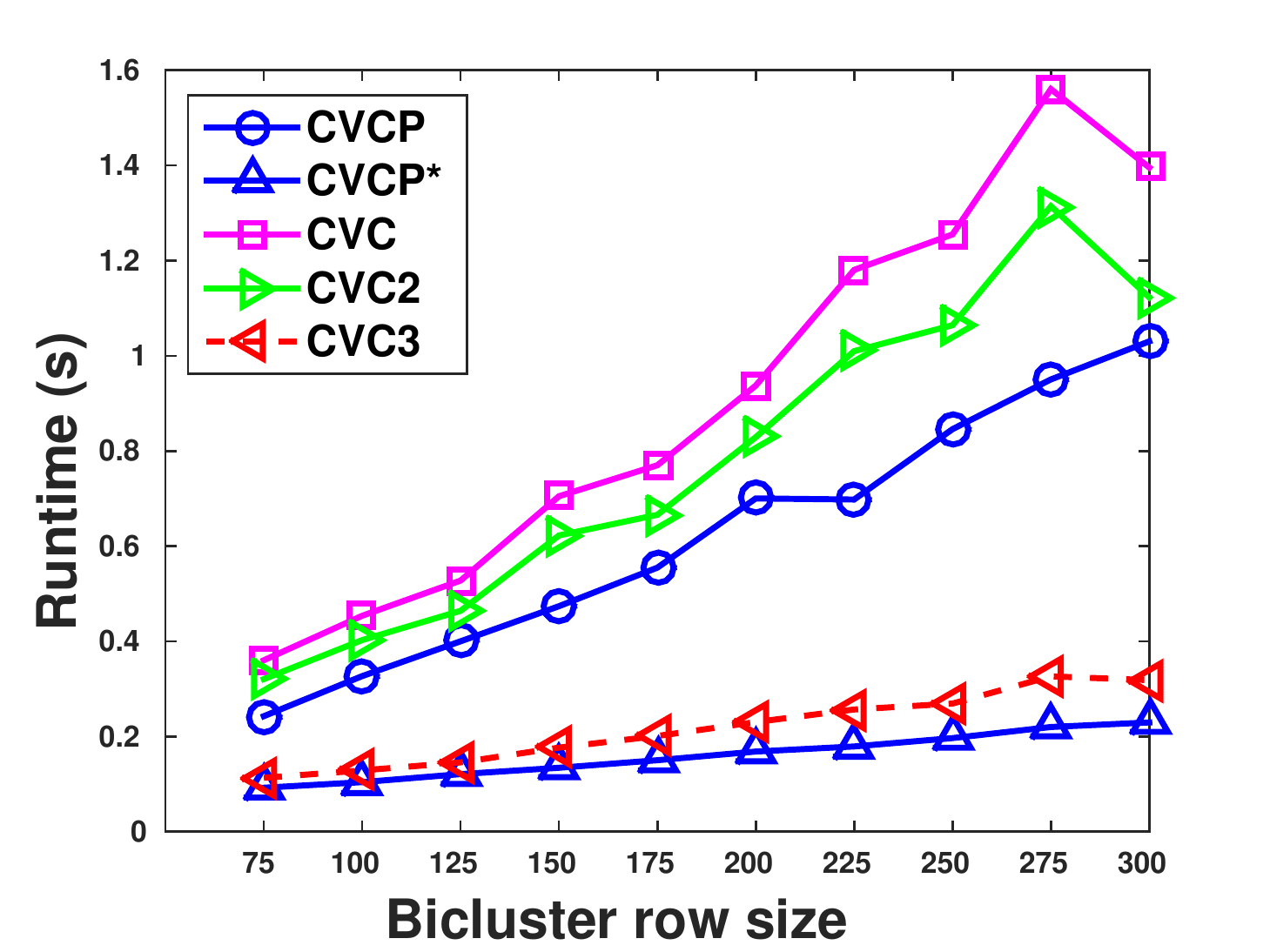}
}
\subfigure[]{
  \includegraphics[trim=0.2cm 0.1cm 0.6cm 0.4cm, clip, scale=0.3]{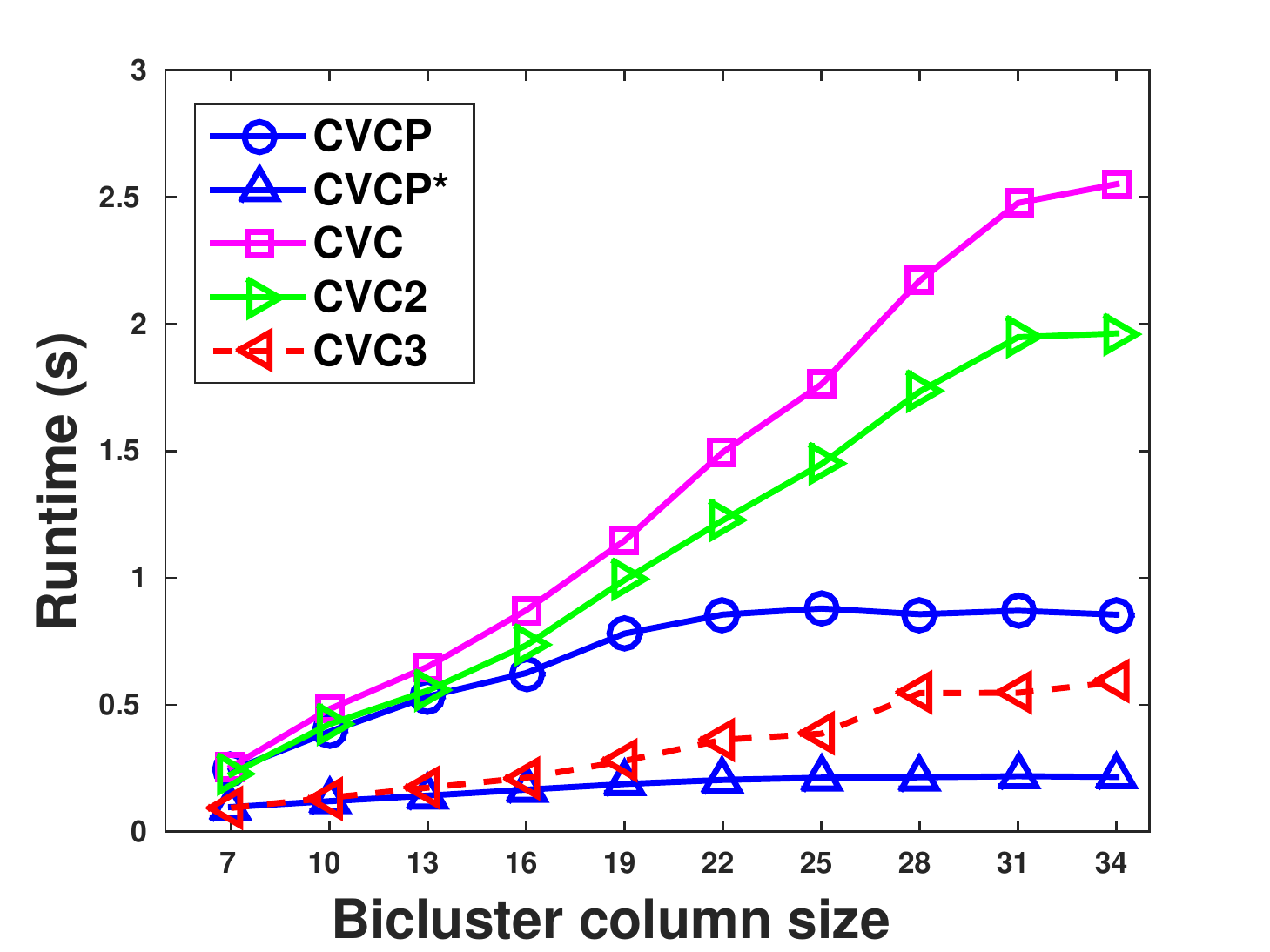}
}
\subfigure[]{
  \includegraphics[trim=0.2cm 0.08cm 0.6cm 0.4cm, clip, scale=0.3]{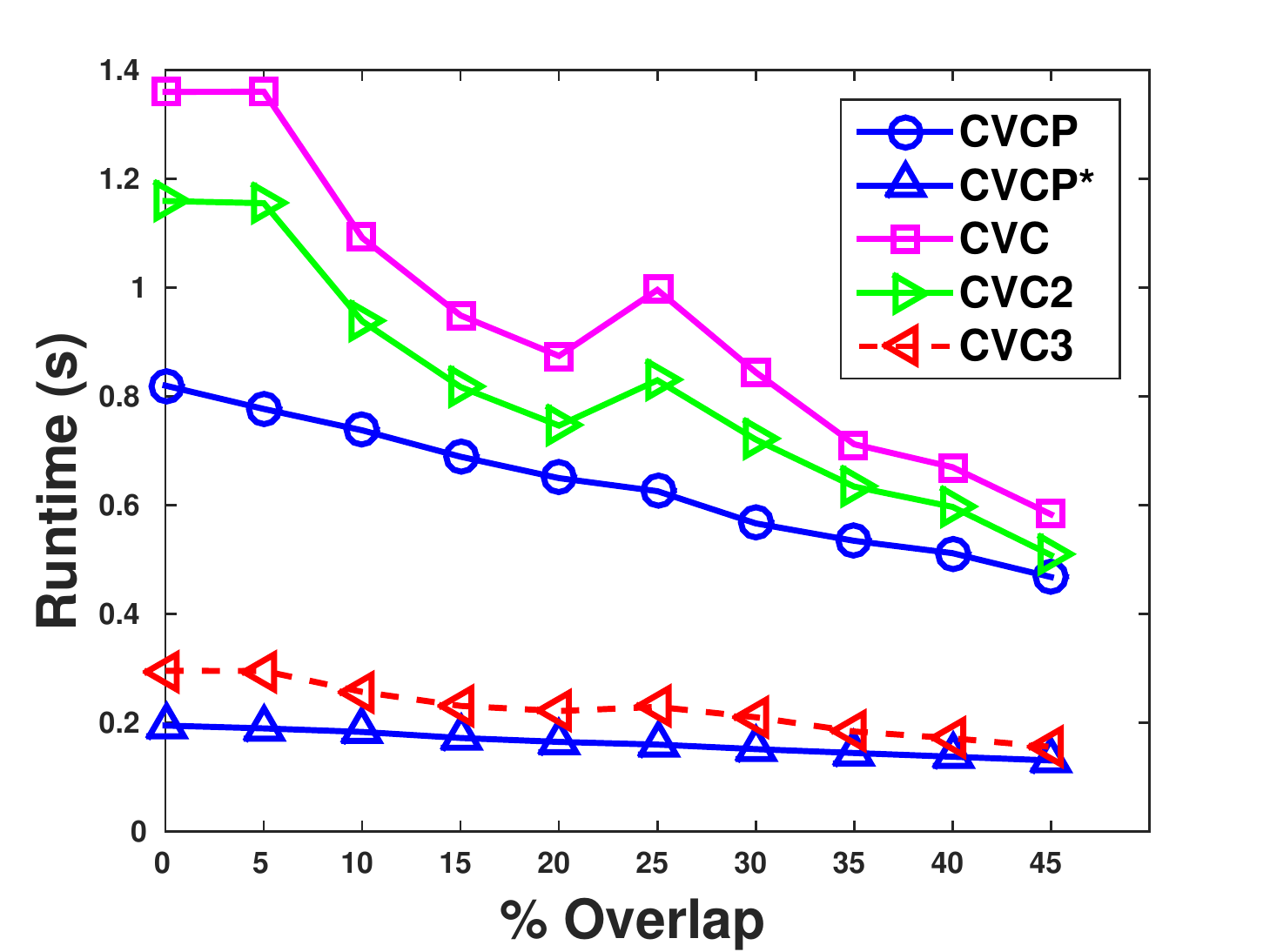}
}
\subfigure[]{
  \includegraphics[trim=0.2cm 0.1cm 0.6cm 0.4cm, clip, scale=0.3]{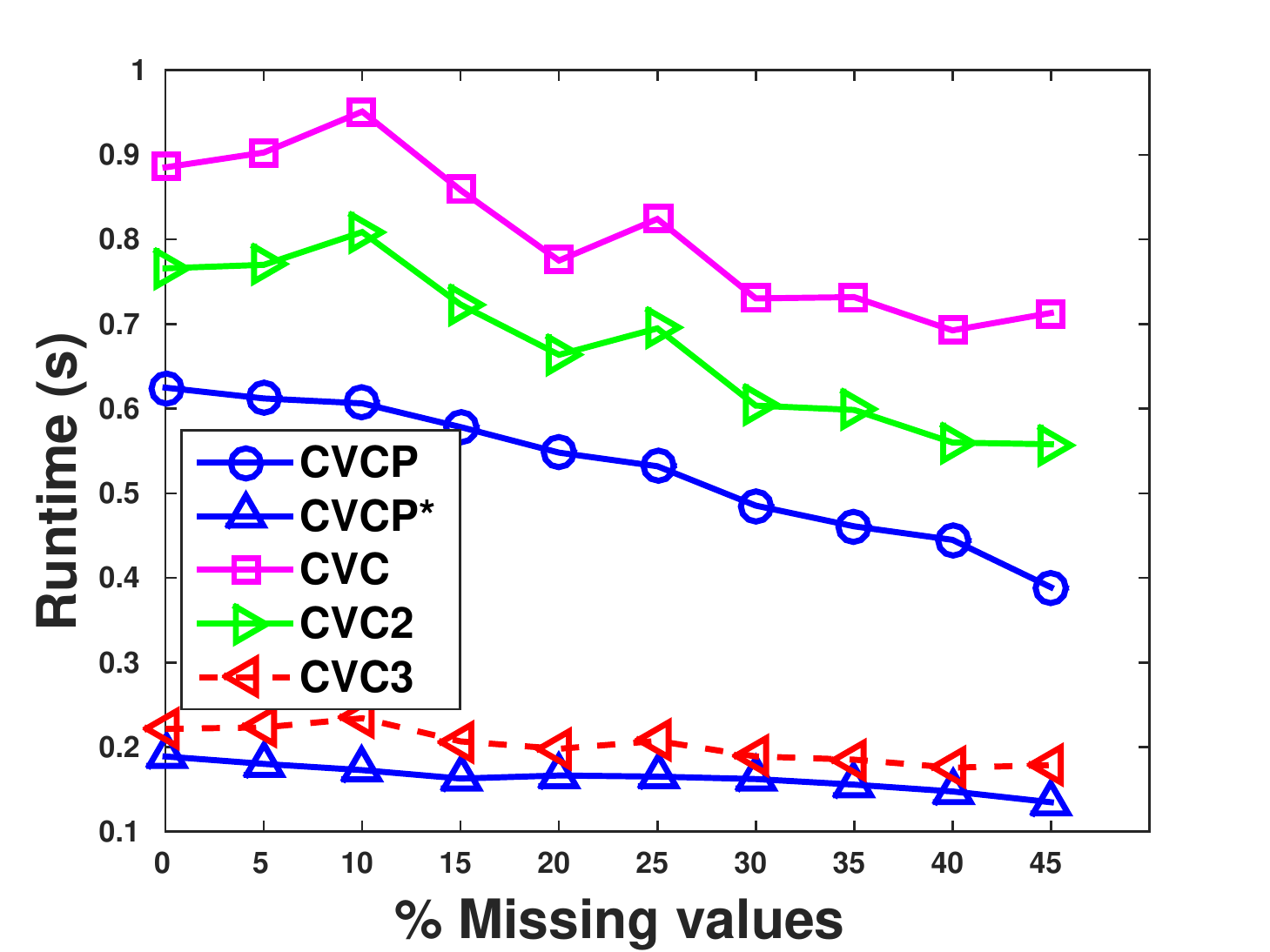}
}
\subfigure[]{
  \includegraphics[trim=0.2cm 0.1cm 0.6cm 0.35cm, clip, scale=0.3]{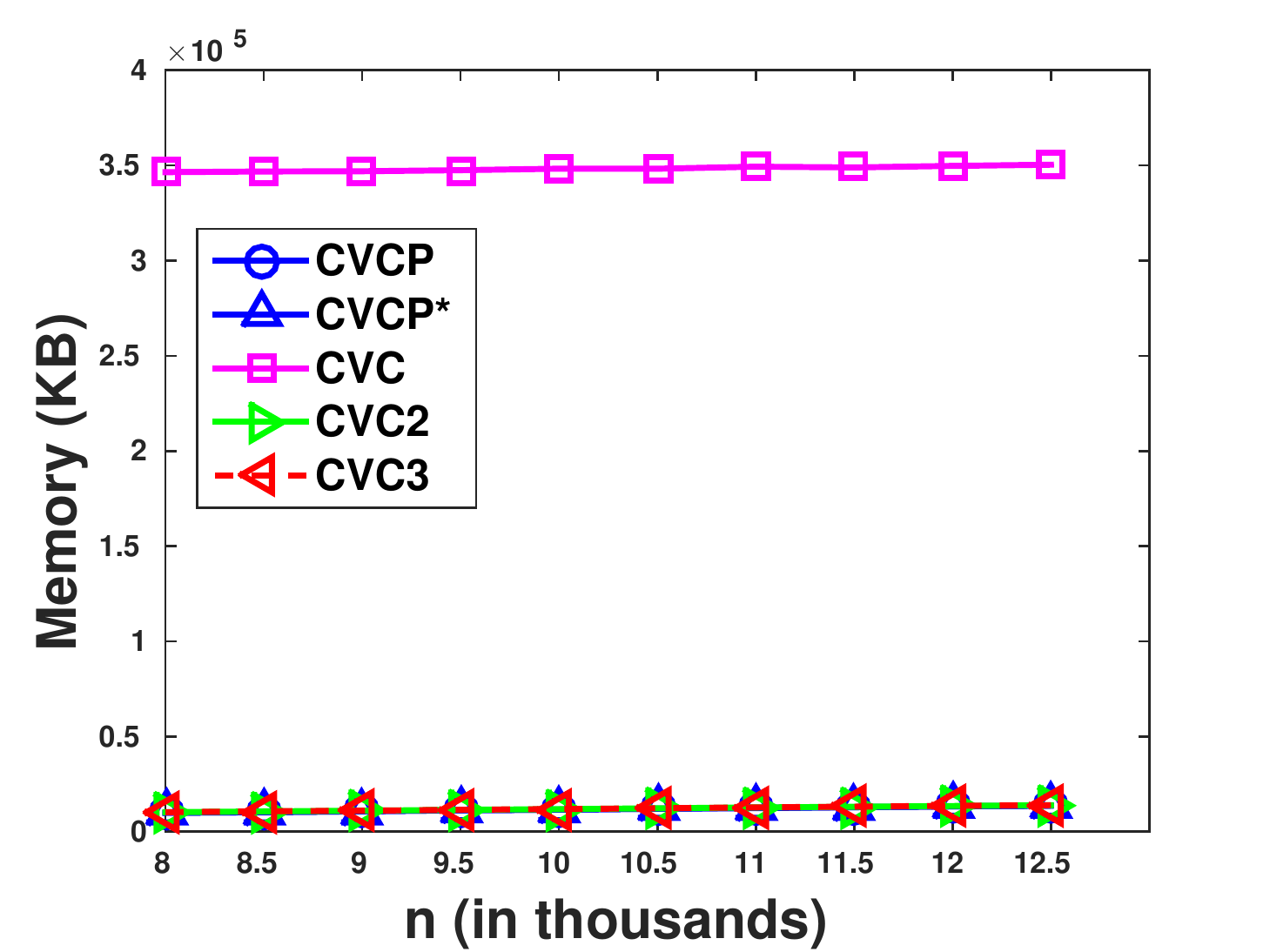}
}
\caption{Runtime of RIn-Close\_CVCP, RIn-Close\_CVCP* (RIn-Close\_CVCP updated with the contributions proposed in \cite{Andrews2017, Andrews2018}), RIn-Close\_CVC, RIn-Close\_CVC2, and RIn-Close\_CVC3 when varying (a) the number $n$ of rows of the dataset, (b) the number $m$ of columns of the dataset, (c) the number of biclusters in the dataset, (d) the bicluster row size, (e) the bicluster column size, (f) the overlap among the biclusters, and (g) the percentage of missing values. (h) Memory usage of these algorithms when varying the number $n$ of rows of the dataset.}
\label{fig:expSynDataRT}
\end{figure}


\subsection{Sensitivity analysis}

This experiment aims to test the runtime and memory usage when real-world datasets are considered, as well as the coverage of the biclustering solutions. For this, we used three gene expression datasets, and tested the algorithms' sensitivity to the parameter $\epsilon$. This parameter defines the maximum perturbation allowed in the biclusters and, therefore, it is related to the coverage and the number of biclusters that will be found in a dataset \cite{VeronezeEtAl2017}. In turn, the number of biclusters in the solution is strongly related to the runtime.

Table~\ref{tab:realdatasets} shows the main properties of the three real-world datasets that we used in our experiments. All of them was downloaded from the NCBI repository \footnote{\url{https://www.ncbi.nlm.nih.gov/}, last access April 1st, 2019}. The attributes of the datasets have a skewed distribution, so we took the logarithm of the values of these datasets. We also scaled the data of each column to real-values between 0 and 1, which enables the usage of the same value of $\epsilon$ for all attributes of a dataset. For simplicity, we also considered only 3 decimal places and multiplied the values by 1000. Thus, the values of the maximum perturbation $\epsilon$ are presented as integers.

We ran the algorithms for 50 times to compute the average runtime and memory usage. We looked for biclusters with at least 3 columns, and Table~\ref{tab:realdatasets} shows the minimum number of rows, $minRow$, used for each dataset. The parameter $minRow$ was set as 5\% of the number $n$ of rows of each dataset. We compared the efficiency of RIn-Close\_CVC algorithms with their best contender, PPS (see Section~\ref{sec:RelWorks} for details). Forcing all algorithms to yield exactly the same solution, the redundancy produced by PPS was suppressed. Our PPS's implementation is available at \url{https://sourceforge.net/projects/pps-cpp/}.

\begin{table}[!htb]
\centering
\small
\caption{Description of the real-world datasets.}
\begin{tabular}{lrr}
\toprule
Name	& Dimension	& $minRow$ \\
\midrule
GDS750	 & $6091  \times 13$ & 305 \\
GDS3035 & $9335  \times 48$ & 467 \\
GDS181  & $12626 \times 84$ & 631 \\
\bottomrule
\end{tabular}
\label{tab:realdatasets}
\end{table}

Figures~\ref{fig:expRealDataNBIC}, \ref{fig:expRealDataRT}, \ref{fig:expRealDataMEM} and \ref{fig:expRealDataCov} show, respectively, the number of biclusters, runtime, memory usage and dataset's coverage for this experiment. To reduce the time for experiments, we limit the runtime of each execution to one hour, so we do not have the PPS results in all tested scenarios.

\begin{figure}[!htb]
\centering
\subfigure[]{
  \includegraphics[trim=0.4cm 0.1cm 1.2cm 0.15cm, clip, scale=0.35]{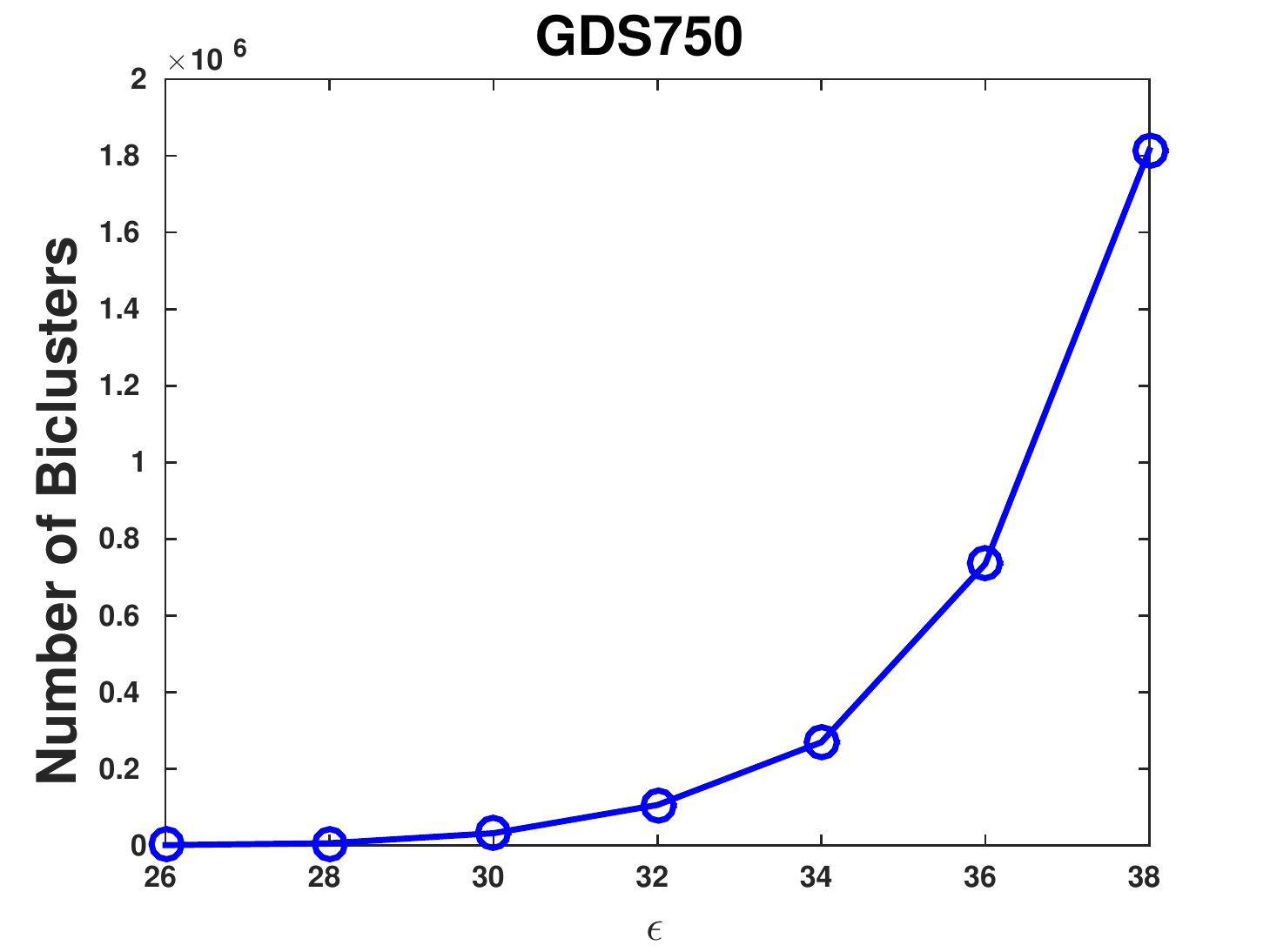}
}
\subfigure[]{
  \includegraphics[trim=0.4cm 0.1cm 1.2cm 0.15cm, clip, scale=0.35]{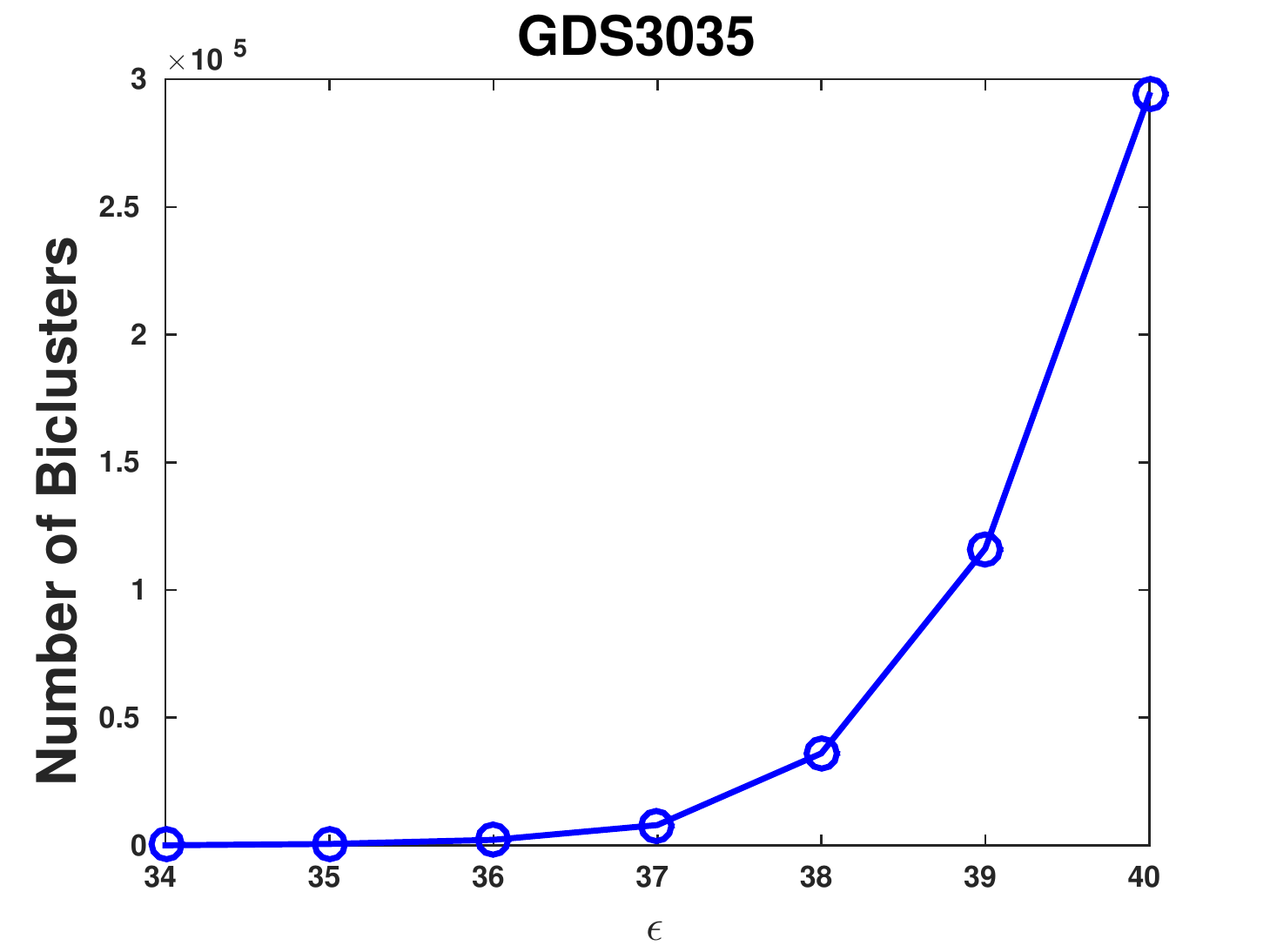}
}
\subfigure[]{
  \includegraphics[trim=0.4cm 0.1cm 1.2cm 0.15cm, clip, scale=0.38]{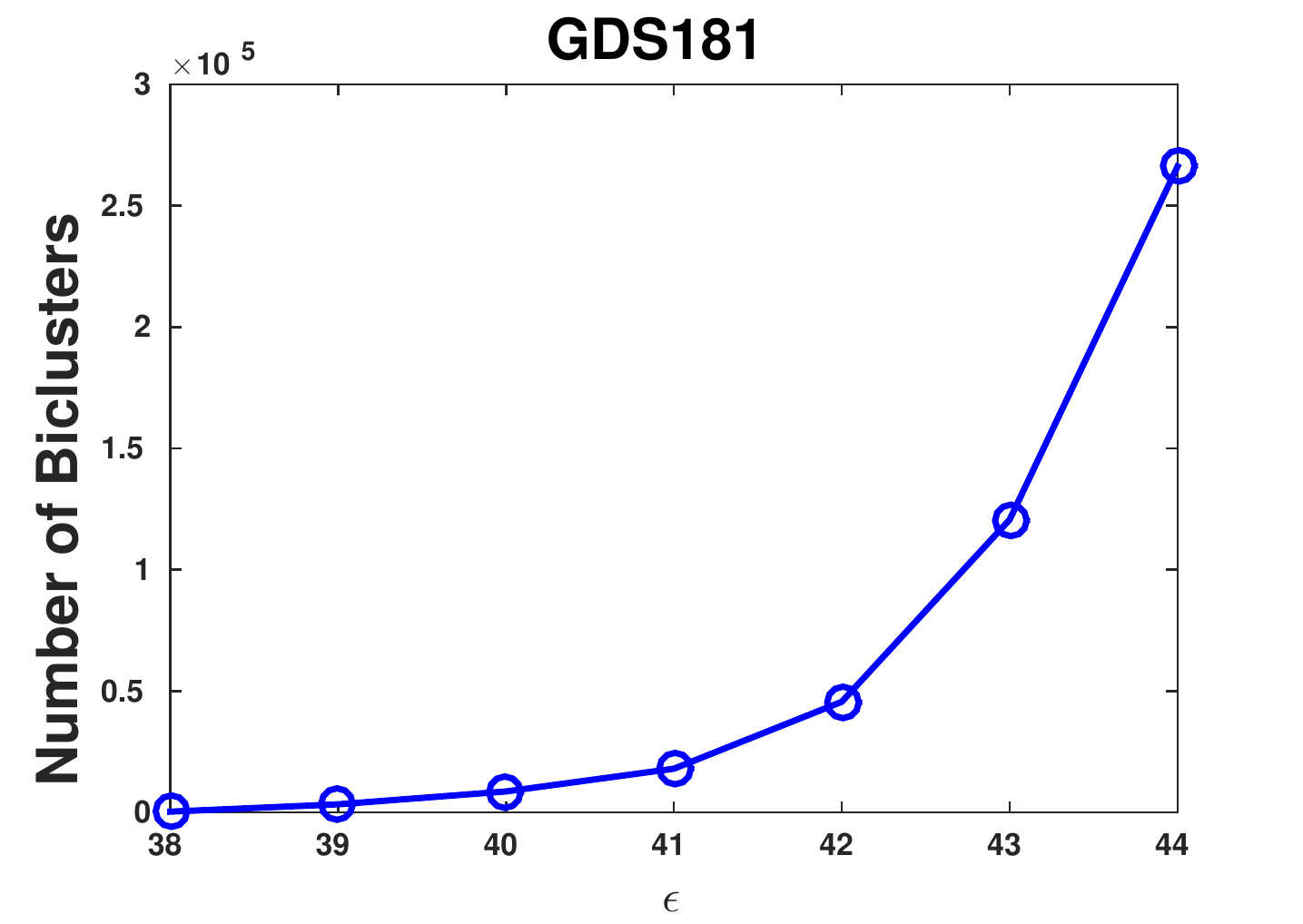}
}
\caption{Number of biclusters when varying the user-defined parameter $\epsilon$ (which controls the maximum perturbation of the biclusters) for the datasets (a) GDS750, (b) GDS3035, and (c) GDS181.}
\label{fig:expRealDataNBIC}
\end{figure}

\begin{figure}[!htb]
\centering
\subfigure[]{
  \includegraphics[trim=0.15cm 0.1cm 1.2cm 0.15cm, clip, scale=0.33]{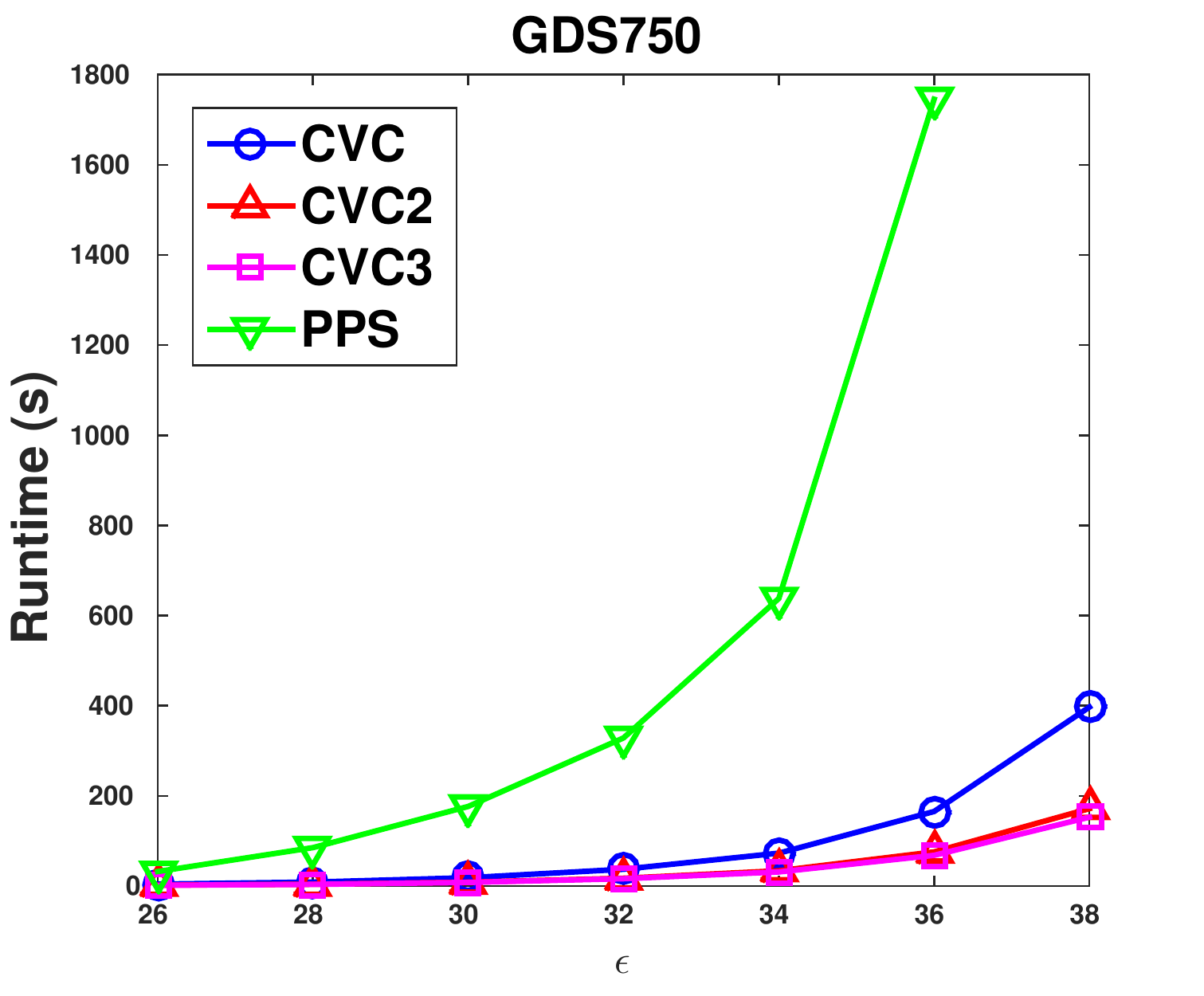}
}
\subfigure[]{
  \includegraphics[trim=0.15cm 0.1cm 1.2cm 0.15cm, clip, scale=0.35]{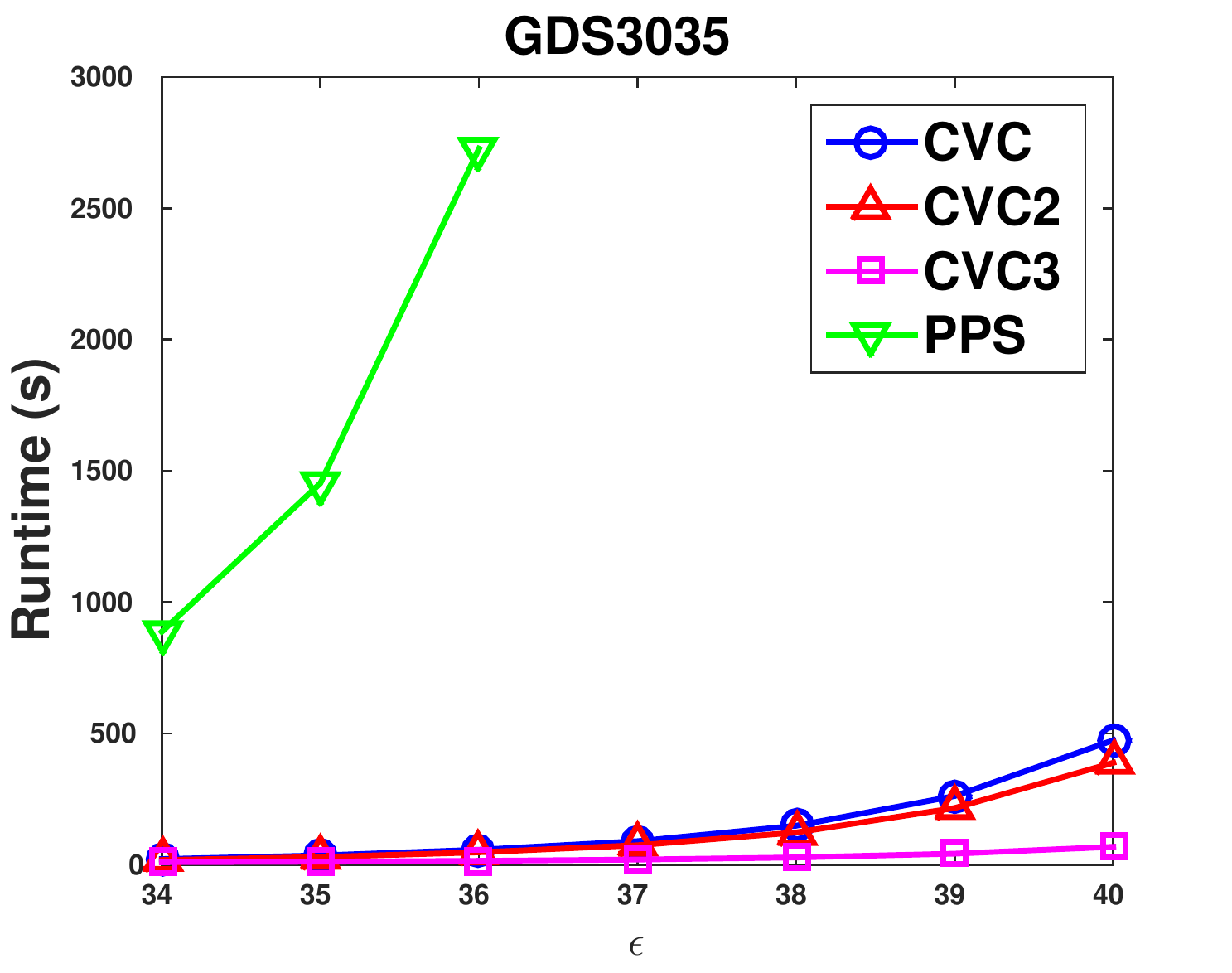}
}
\subfigure[]{
  \includegraphics[trim=0.15cm 0.1cm 1.2cm 0.15cm, clip, scale=0.38]{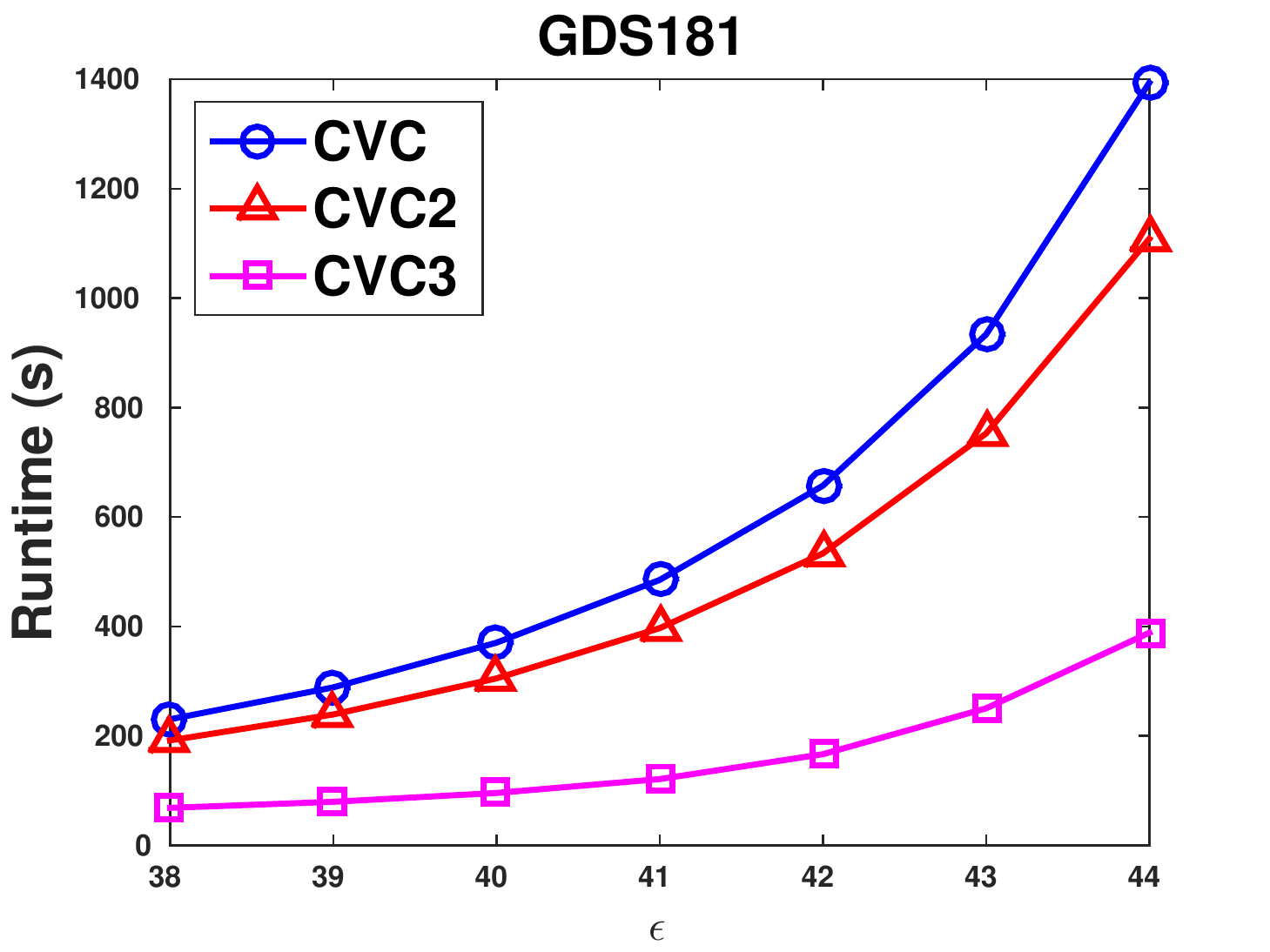}
}
\caption{Runtime when varying the user-defined parameter $\epsilon$ (which controls the maximum perturbation of the biclusters) for the datasets (a) GDS750, (b) GDS3035, and (c) GDS181.}
\label{fig:expRealDataRT}
\end{figure}

\begin{figure}[!htb]
\centering
\subfigure[]{
  \includegraphics[trim=0.4cm 0.1cm 1.2cm 0.15cm, clip, scale=0.33]{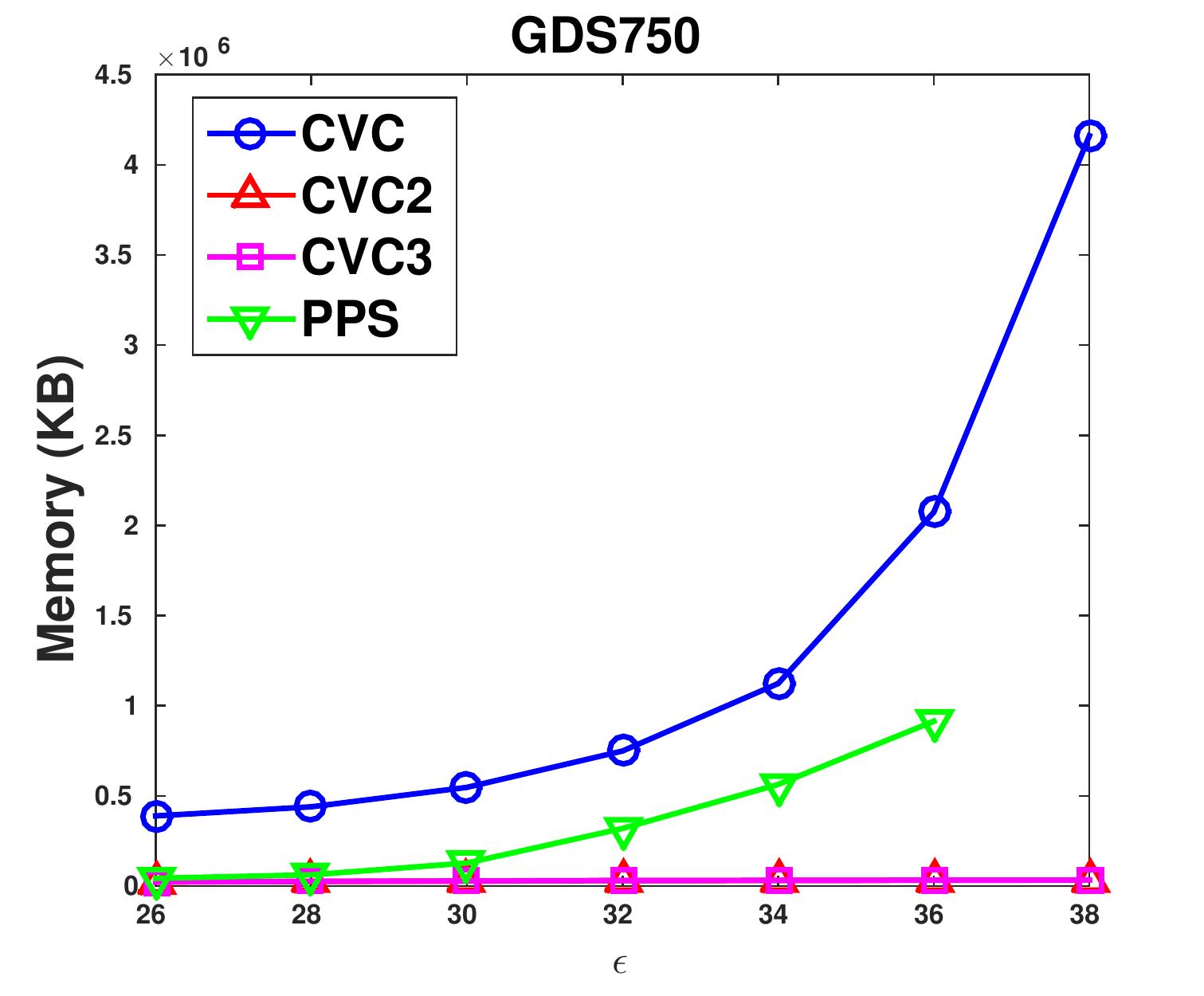}
}
\subfigure[]{
  \includegraphics[trim=0.4cm 0.1cm 1.2cm 0.15cm, clip, scale=0.35]{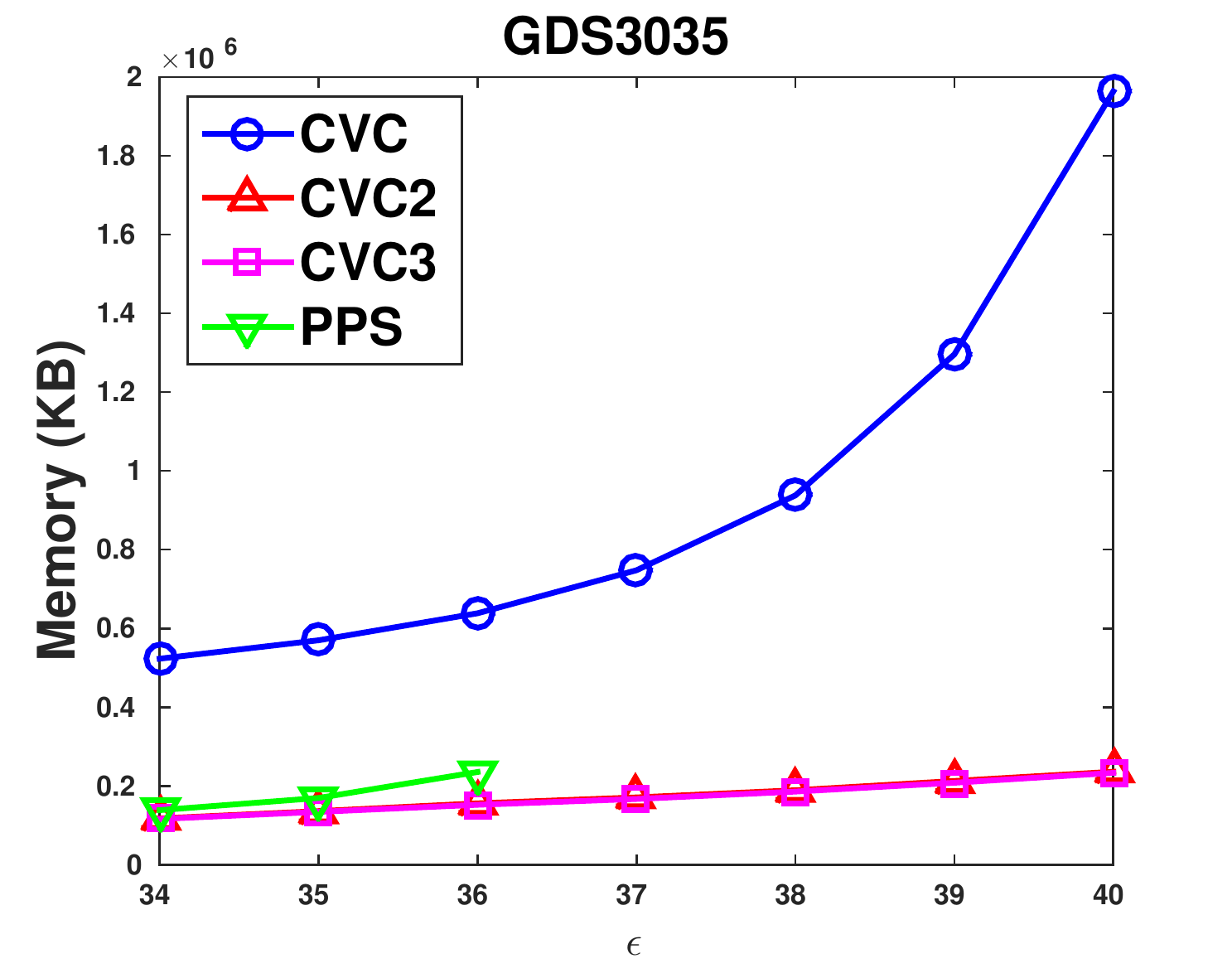}
}
\subfigure[]{
  \includegraphics[trim=0.4cm 0.1cm 1.2cm 0.15cm, clip, scale=0.38]{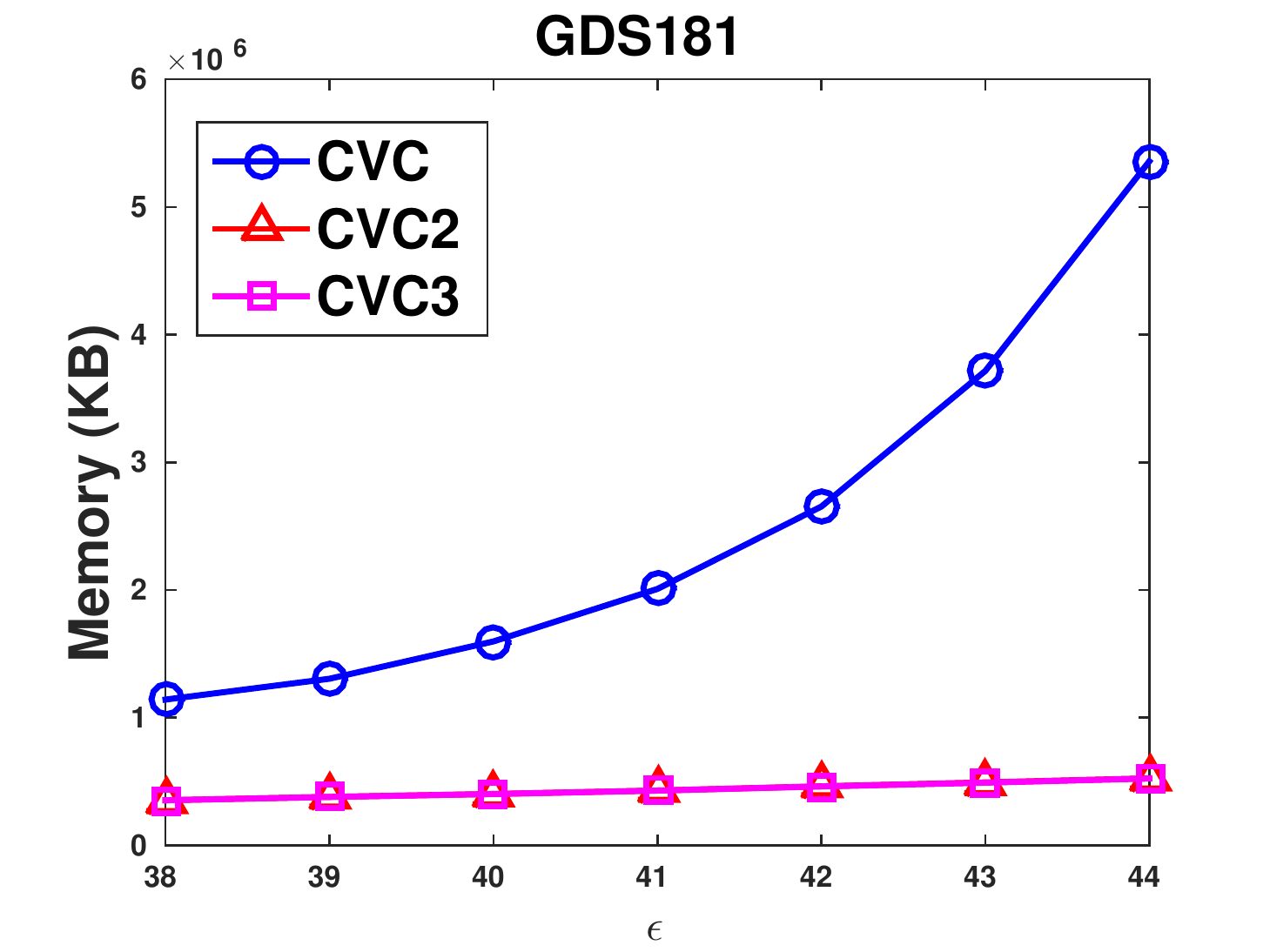}
}
\caption{Memory usage when varying the user-defined parameter $\epsilon$ (which controls the maximum perturbation of the biclusters) for the datasets (a) GDS750, (b) GDS3035, and (c) GDS181.}
\label{fig:expRealDataMEM}
\end{figure}

\begin{figure}[!htb]
\centering
\subfigure[]{
  \includegraphics[trim=0.4cm 0.1cm 1.2cm 0.15cm, clip, scale=0.35]{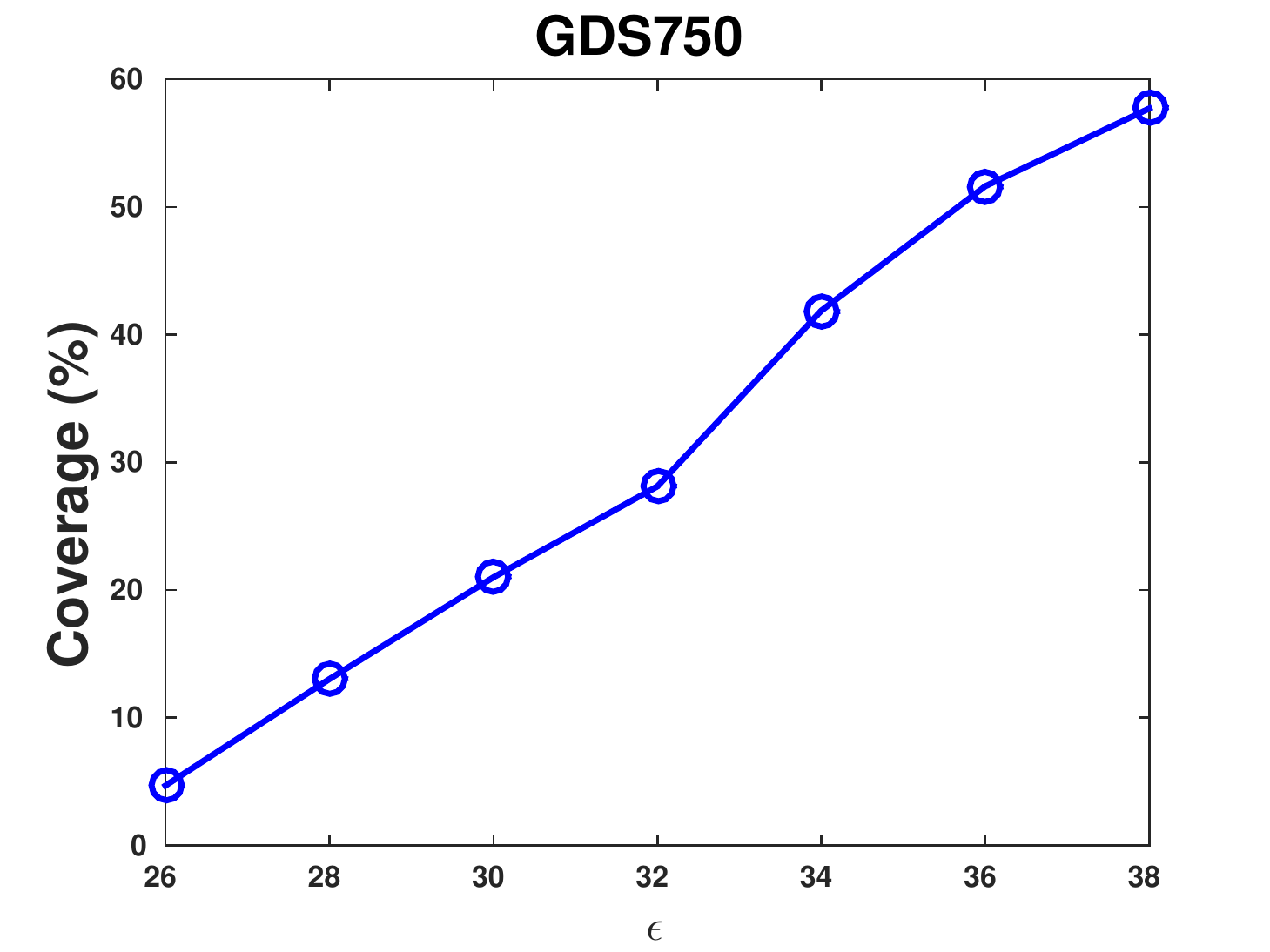}
}
\subfigure[]{
  \includegraphics[trim=0.4cm 0.1cm 1.2cm 0.15cm, clip, scale=0.35]{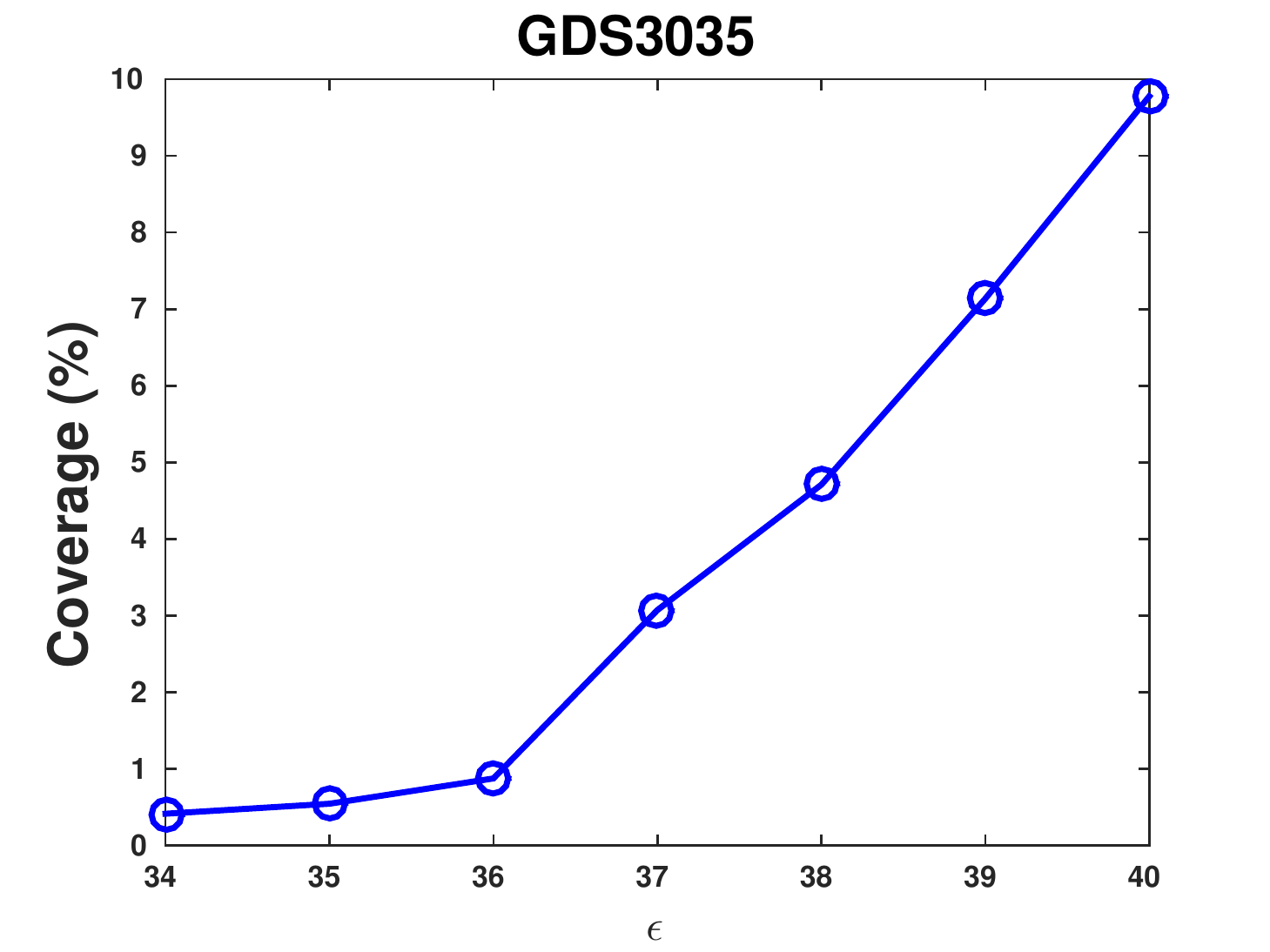}
}
\subfigure[]{
  \includegraphics[trim=0.4cm 0.1cm 1.2cm 0.15cm, clip, scale=0.35]{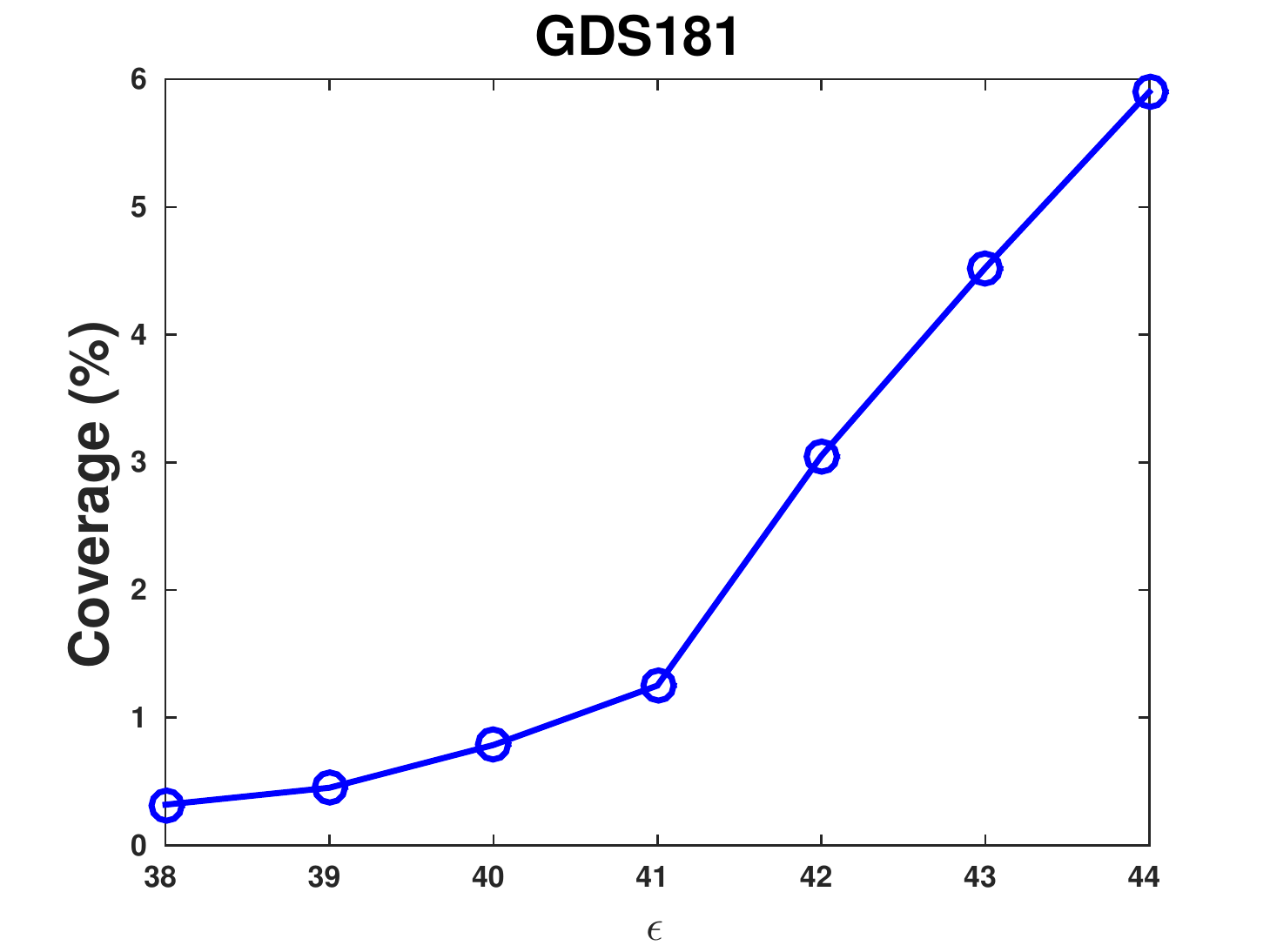}
}
\caption{Coverage of the biclustering solution when varying the user-defined parameter $\epsilon$ (which controls the maximum perturbation of the biclusters) for the datasets (a) GDS750, (b) GDS3035, and (c) GDS181.}
\label{fig:expRealDataCov}
\end{figure}

As expected, the number of biclusters increased exponentially with the increase in the value of $\epsilon$ for all datasets. The gain of RIn-Close\_CVC2 and RIn-Close\_CVC3 in the memory usage is quite expressive. They presented a linear growth in the memory usage, even though the number of biclusters exhibited an exponential growth with the value of $\epsilon$, which is in contrast with the first version of RIn-Close\_CVC. RIn-Close\_CVC2 obtained a better runtime than RIn-Close\_CVC in all scenarios, and RIn-Close\_CVC3 presented an even better runtime, specially for the datasets with more attributes. The higher the number of attributes, the greater the number of canonicity tests tends to be (and, consequently, inherited canonicity test failures). When comparing RIn-Close\_CVC algorithms against its best competitor, PPS, we notice that all versions of RIn-Close\_CVC are better than PPS in terms of runtime, but PPS is better than the first version of RIn-Close\_CVC in terms of memory usage. The coverage is a monotonically non-decreasing function of the value of $\epsilon$ \cite{VeronezeEtAl2017}. However, as we can see, the coverage does not grow as fast as the number of biclusters for these datasets. It happened due to bicluster overlapping \cite{VeronezeEtAl2017}: indicating that we are finding new biclusters in portions of the dataset explored with lower values of $\epsilon$. Remarkably, the number of biclusters will start to decrease when $\epsilon$ achieves a sufficiently high value and, indeed, the entire dataset will be considered a single bicluster if the value of $\epsilon$ is high enough \cite{VeronezeEtAl2017}. Those scenarios are not explored here and we invite the reader to see more details about the relationship between number of biclusters, coverage and overlap in our previous work \cite{VeronezeEtAl2017}.

We also performed a comparison between the dataset's coverage using online partitioning, such as done by RIn-Close\_CVC, and a priori partitioning.
For that, we set the maximum perturbation $\epsilon$ of RIn-Close\_CVC and the bin-width of the partitioning equally. We used RIn-Close\_CVCP to mine the biclusters in the partitioned datasets (but we could have itemized the datasets and used some traditional FCA / FPM algorithm as well, see \ref{sec:appendix_class1x2} for details). Table~\ref{tab:resOnAp1} shows the results. As expected, the obtained coverage with the online partitioning is much better than the obtained coverage with a priori partitioning, even resorting to enumerative biclustering algorithms in both cases. 

\begin{table}[!htb]
\centering
\small
\caption{Comparison of the dataset's coverage using online and a priori partitioning. The bin-width used in the a priori partitioning was equal to the value of maximum perturbation $\epsilon$ used in the online partitioning.}
\begin{tabular}{l|c|r|r}
\toprule
& & \multicolumn{2}{c}{Coverage (\%)}\\
Dataset	& $\epsilon$ & Online & A priori \\
\midrule
GDS750	& 40 & 62.00 & 31.36 \\
GDS3035 & 40 &  9.78 &  1.41 \\
GDS181  & 50 & 17.51 &  3.02 \\
\bottomrule
\end{tabular}
\label{tab:resOnAp1}
\end{table}

\subsection{Supervised descriptive pattern mining}

This experimental analysis first provides a comparison between biclustering enumerative solutions with a priori and online partitioning of the numerical attributes of a mixed-attribute dataset. We focused on labeled datasets with the main aim of mining QCARs, i.e, to find a set of attributes together with their values that properly describe the target variable. Thereby, we show how the biclusters yield a parsimonious set of relevant rules, providing useful and interpretable models for a dataset.  Table~\ref{tab:datasets} describes the mixed-attribute datasets used in our experiments. All of them were downloaded from the UCI Repository \footnote{\url{http://archive.ics.uci.edu/ml/}, last access April 1st, 2019}. The dataset Acute has two decision variables, and we will provide results for both of them.

\begin{table}[!htb]
\caption{Description of the mixed-attribute datasets.}
\centering
\small
\begin{tabular}{rlcccl}
\toprule
\textbf{\#} & \textbf{Name} & \textbf{\# rows} & \textbf{\# columns} & \textbf{\# labels} & \textbf{Attributes' description}\\
\midrule
1 & Acute & 120  & 6  & 2  & 1 numerical, 5 categorical\\
2 & Heart  & 270  & 13 & 2  & 5 numerical, 8 categorical \\
3 & Australian  & 690  & 14 & 2  & 6 numerical, 8 categorical \\
\bottomrule
\end{tabular}
\label{tab:datasets}
\end{table}

We use the following equal-width binning rules to partition the numerical attributes: Scott, Freedman-Diaconis (FD), Sturges, and Square Root (SQRT). All of them are provided by the binning functions of MATLAB, Python, and others. We use RIn-Close\_CVCP to enumerate the biclusters in the partitioned datasets. For the numerical attributes, we also use these binning rules to provide the value of the parameter $\epsilon$ of RIn-Close\_CVC3, and we set $\epsilon = 0$ for the categorical attributes. We set $minRow = 5$ and $minCol= 1$, for both RIn-Close\_CVC3 and RIn-Close\_CVCP. A bicluster was considered an interesting QCAR if its confidence was at least 0.95, and its minimum distance from 1 for the lift metric was at least 0.2 \cite{VenturaLuna2018sdpm}.

\begin{table}[!htb]
\caption{Number of biclusters of the solutions with a priori and online partitioning of the numerical attributes.}
\centering
\small
\begin{tabular}{clrrr}
\toprule
& & \textbf{Acute} & \textbf{Heart} & \textbf{Australian} \\
\midrule
\multirow{2}{*}{\textbf{FD}} & \textbf{A priori} & 79 & 12,661 & 50,878 \\
                             & \textbf{Online} & 205 & 58,826 & 299,751 \\
\hline
\multirow{2}{*}{\textbf{Scott}} & \textbf{A priori} & 79 & 15,048 & 66,902 \\
                                & \textbf{Online} & 205 & 90,649 & 553,704 \\
\hline
\multirow{2}{*}{\textbf{SQRT}} & \textbf{A priori} & 94 & 11,811 & 48,896 \\
                               & \textbf{Online} & 229 & 51,259 & 272,974 \\
\hline
\multirow{2}{*}{\textbf{Sturges}} & \textbf{A priori} & 79 & 16,554 & 72,649 \\
                                  & \textbf{Online} & 205 & 116,990 & 1,025,160 \\
\bottomrule
\end{tabular}
\label{tab:resul_qtdbic}
\end{table}

Table~\ref{tab:resul_qtdbic} shows the number of biclusters of the solutions with online and a priori partitioning of the numerical attributes. We notice that the number of biclusters of the results with online partitioning is always greater than the number of biclusters of the results with a priori partitioning. As we know, when we use an equal-width binning to determine the arity of the partitioning (equal to the value of the parameter $\epsilon$ of RIn-Close\_CVC3), enumerative biclustering solutions with a priori partitioning is contained in the enumerative biclustering solutions with online partitioning (missing some biclusters, and finding only parts of some others). Given that the gross number of obtained biclusters is high, it seems reasonable to apply some filter to detect only relevant biclusters. So, it remains to know if the number of relevant biclusters is also greater when not using a priori partitioning. Table~\ref{tab:resul_filters} shows the number of interesting QCARs selected from the enumerative solutions. For the Acute dataset, we are presenting the results for the two decision variables, D1 and D2, respectively. As we can see, there is loss of relevant biclusters with the a priori partitioning of the numerical attributes, leading to a lower row-coverage for the Heart and Australian datasets, which have more numerical attributes than the Acute dataset (which has only 1 numerical attribute).

\begin{table}[!htb]
\caption{Biclustering results after the selection of interesting QCARs.}
\centering
\small
\setlength\tabcolsep{1pt}
\begin{tabular}{cc|rr|rr|rr|rr}
\toprule
& & \multicolumn{2}{c}{ \textbf{FD} } & \multicolumn{2}{c}{ \textbf{Scott} } & \multicolumn{2}{c}{ \textbf{SQRT} } & \multicolumn{2}{c}{ \textbf{Sturges} } \\
& & \textbf{A priori}. & \textbf{Online} & \textbf{A priori.} & \textbf{Online} & \textbf{A priori.} & \textbf{Online} & \textbf{A priori.} & \textbf{Online} \\
\midrule
\multirow{3}{*}{ \textbf{Acute - D1} } & \textbf{\# of QCARs} & 40 & 86 & 40 & 86 & 44 & 93 & 40 & 86 \\
& \textbf{\%row-coverage} & 100.00 & 100.00 & 100.00 & 100.00 & 100.00 & 100.00 & 100.00 & 100.00 \\
\hline
\multirow{3}{*}{ \textbf{Acute - D2} } & \textbf{\# of QCARs} & 52 & 127 & 52 & 127 & 64 & 157 & 52 & 127 \\
& \textbf{\%row-coverage }& 100.0 & 100.0 & 100.0 & 100.0 & 100.0 & 100.0 & 100.0 & 100.0 \\
\hline
\multirow{3}{*}{ \textbf{Heart} } & \textbf{\# of QCARs} & 2809 & 11294 & 3337 & 17155 & 2637 & 9958 & 3555 & 21822 \\
& \textbf{\%row-coverage} & 98.15 & 99.63 & 98.89 & 99.63 & 97.04 & 98.89 & 98.89 & 99.63 \\
\hline
\multirow{3}{*}{ \textbf{Australian} } & \textbf{\# of QCARs} & 9844 & 55877 & 12980 & 88367 & 9276 & 43891 & 11830 & 95056 \\
& \textbf{\%row-coverage} & 97.25 & 98.26 & 97.54 & 98.70 & 97.10 & 98.70 & 96.09 & 98.26 \\
\bottomrule
\end{tabular}
\label{tab:resul_filters}
\end{table}

Table~\ref{tab:rules} shows some rules for the Acute (for both the 1st and 2nd decision variables), Heart, and Australian datasets. It also shows the completeness (also called recall), confidence, lift and leverage of each rule \cite{VenturaLuna2018sdpm}. We can notice how the CVC biclusters 
can provide a useful and interpretable model for a dataset by means of association rules, specifically QCARs in this case.

For instance, for the Acute dataset, we can observe that almost all the patients with inflammation of urinary bladder presented urine pushing (continuous need for urination) and micturition pains. In the same way, we can observe that almost all the patients with nephritis of renal pelvis origin presented lumbar pain and urine pushing. For the Heart dataset, we can see that 30\% of the patients without heart disease were female, with no exercise induced angina, and 0 major vessel colored by fluoroscopy. At the same time, 23\% of the patients with heart disease were male, with an asymptomatic chest pain, and 1 major vessel colored by fluoroscopy. Unfortunately, the Australian dataset does not have the description of its attributes. Anyway, we can notice that the 8th attribute of this dataset is very important to determine the approval or rejection of a credit card application. On the other hand, the 12th attribute is irrelevant, with 90\% of the customers having it equal to 2. We refrained from presenting all the QCARs and their corresponding interpretation due to space restriction.

\begin{table}[!htb]
\caption{Some rules mined by RIn-Close\_CVC3 for Acute, Heart and Australian datasets.}
\footnotesize
\centering
\small
\begin{tabular}{lp{10cm}rrrr}
\toprule
\multicolumn{6}{c}{\textbf{Acute - 1st decision variable (inflammation of urinary bladder)}} \\
\midrule
\# & \textbf{Rule} & \textbf{Comp.} & \textbf{Conf.} & \textbf{Lift} & \textbf{Lev.} \\
\midrule
1 & nausea\{no\}, lumbarPain\{yes\}, micturitionPain\{no\} $\Rightarrow$ 0  & 0.67 & 1.00  & 1.97  & 0.17 \\ 
2 & urinePushing\{no\}, urethraBurning\{no\} $\Rightarrow$ 0  & 0.66 & 1.00  & 1.97  & 0.16 \\ 
3 & urinePushing\{yes\}, micturitionPain\{yes\} $\Rightarrow$ 1  & 0.83 & 1.00  & 2.03  & 0.21 \\ 
4 & urinePushing\{yes\}, urethraBurning\{no\} $\Rightarrow$ 1  & 0.51 & 1.00  & 2.03  & 0.13 \\ 
\midrule
\multicolumn{6}{c}{\textbf{Acute - 2nd decision variable (nephritis of renal pelvis origin)}} \\
\midrule
\# & \textbf{Rule} & \textbf{Comp.} & \textbf{Conf.} & \textbf{Lift} & \textbf{Lev.} \\
\midrule
1 & nausea\{no\}, lumbarPain\{no\} $\Rightarrow$ 0  & 0.71 & 1.00  & 1.71  & 0.17 \\ 
2 & nausea\{no\}, urethraBurning\{no\} $\Rightarrow$ 0  & 0.71 & 1.00  & 1.71  & 0.17 \\ 
3 & lumbarPain\{yes\}, urinePushing\{yes\} $\Rightarrow$ 1  & 0.80 & 1.00  & 2.40  & 0.19 \\
4 & nausea\{yes\}, lumbarPain\{yes\}, micturitionPain\{yes\} $\Rightarrow$ 1  & 0.58 & 1.00  & 2.40  & 0.14 \\ 
\midrule
\multicolumn{6}{c}{\textbf{Heart (absence or presence of heart disease)}} \\
\midrule
\# & \textbf{Rule} & \textbf{Comp.} & \textbf{Conf.} & \textbf{Lift} & \textbf{Lev.} \\
\midrule
1 & sex\{F\}, exercIAngina\{no\}, vesselsColor\{0\} $\Rightarrow$ 0  & 0.30 & 0.96  & 1.72  & 0.07 \\ 
2 & chol[186.00,236.00], oldpeak[0.00,0.50], vesselsColor\{0\} $\Rightarrow$ 0  & 0.22 & 0.97  & 1.75  & 0.05 \\ 
3 & sex\{M\}, chestPain\{asymptomatic\}, vesselsColor\{1\} $\Rightarrow$ 1  & 0.23 & 0.97  & 2.17  & 0.06 \\ 
4 & sex\{M\}, chol[266.00,315.00], exercIAngina\{yes\} $\Rightarrow$ 1  & 0.19 & 1.00  & 2.25  & 0.05 \\ 
\midrule
\multicolumn{6}{c}{\textbf{Australian (approval or rejection of credit card applications)}} \\
\midrule
\# & \textbf{Rule} & \textbf{Comp.} & \textbf{Conf.} & \textbf{Lift} & \textbf{Lev.} \\
\midrule
1 & att8\{0\}, att12\{2\} $\Rightarrow$ 0  & 0.71 & 0.95  & 1.72  & 0.16 \\ 
2 & att2[29.58,34.58], att8\{0\} $\Rightarrow$ 0  & 0.15 & 0.97  & 1.74  & 0.03 \\ 
3 & att8\{1\}, att12\{2\}, att14[397,2385] $\Rightarrow$ 1  & 0.28 & 0.96  & 2.15  & 0.07 \\ 
4 & att4\{2\}, att8\{1\}, att12\{2\}, att14[197,2185] $\Rightarrow$ 1  & 0.29 & 0.96  & 2.15  & 0.07 \\ 
\bottomrule
\end{tabular}
\label{tab:rules}
\end{table}

\section{Concluding remarks}
\label{sec:conclusion}

Concerning enumerative biclustering of numerical datasets, we proposed here a new version of RIn-Close\_CVC, named RIn-Close\_CVC3, which brings a large reduction in memory usage, and also a significant runtime gain. We also showed that the new algorithm also keeps the four key properties of its predecessor: efficiency, completeness, correctness and non-redundancy. Moreover, RIn-Close\_CVC3 can handle datasets: (1) with missing values, (2) characterized by attributes with distinct distributions, and (3) characterized by mixed data types.

The results of our experiments with synthetic datasets showed that the memory usage of RIn-Close\_CVC3 was equivalent to the memory usage of RIn-Close\_CVCP, known to be very attractive. Also, the results of our experiments with real-world datasets showed that RIn-Close\_CVC3 presented a linear growth in the memory usage, even though the number of biclusters exhibited an exponential growth with the value of admissible residue $\epsilon$ (for sufficiently higher values of $\epsilon$, the exponential growth will cease and the number of biclusters will converge to a single one). Thus, RIn-Close\_CVC3 promotes a breakthrough in terms of memory usage by extending the use of lexicographic order for both rows and columns. In this way, a symbol table in memory is no more required and redundancy is avoided by further exploring the lexicographic order. Moreover, the new version also achieved better results in terms of runtime in our experiments. Therefore, the new version dominates its predecessor in all relevant scalability aspects. These are great achievements, opening the possibility of enumerating perturbed maximal CVC biclusters in more challenging scenarios.

Our experimental results also outlined the clear advantages of an online partitioning, as done by RIn-Close CVC3, over an a priori partitioning. Notice that a priori partitioning is a necessary step for biclustering approaches based on traditional enumerative algorithms of FCA / FPM, such as In-Close or LCM. Moreover, our experimental results indicated that the QCARs extracted from the biclusters produced by RIn-Close\_CVC3 are valuable and automatic means of providing useful and relevant interpretable models of a dataset, yielding a parsimonious set of relevant rules for discriminating the class labels.

In future works, the interplay between RIn-Close\_CVC3 biclustering algorithm and QCARs will be further explored in the context of supervised descriptive pattern mining, and also associative classification. We also intend to further explore data-driven parameterization for both a priori and online partitioning enumerative biclustering, making the parameter setting more user-friendly.

\section*{Acknowledgments}

R. Veroneze and F. J. Von Zuben would like to thank FAPESP (Process Number: 2017/21174-8), CAPES and CNPq (Process Number 307228/2018-5) for the financial support.

\appendix

\section{Example of how to mine biclusters in a mixed-attribute dataset}
\label{sec:appendix_ma}

\setcounter{table}{0}

A dataset may have \emph{numerical} and \emph{categorical} attributes. The numerical attributes can be \emph{discrete} (integer attributes) or \emph{continuous} (real-valued attributes). The categorical attributes can be \emph{ordinal} or \emph{nominal}. Binary attributes can be seen as nominal attributes that can assume only two values, such as \emph{Yes} or \emph{No}. A \emph{mixed-attribute dataset} is a dataset with a single type per column and possibly distinct attribute types along the columns. Table~\ref{tab:maDataEx1} shows an example of a mixed-attribute dataset. The attributes \emph{Sex}, \emph{Smoker}, and \emph{Preferred Movie Genre} are nominal attributes, with \emph{Sex} and \emph{Smoker} being binary attributes. \emph{Social Class} is an ordinal attribute, where label \emph{A} represents the wealthiest people. \emph{Age} is taken as an integer attribute. \emph{Weight} and \emph{Height} are considered real-valued attributes.

\begin{table}[!htb]
\caption{Example of a mixed-attribute dataset, with two perturbed biclusters highlighted.}
\centering
\small
\setlength\tabcolsep{1.5pt}
\begin{tabular}{cccrrccc}
\toprule
\textbf{\#} & \textbf{Sex} & \textbf{Age} & \textbf{Weight (kg)} & \textbf{Height (m)} & \textbf{Smoker} & \textbf{Preferred} & \textbf{Social Class} \\
 & & & & & & \textbf{Movie Genre} & \\
\midrule
1   & \colorbox[rgb]{0.7,0.7,0.7}{F}            & 32           & 94.87                & 1.72                & Y               & \colorbox[rgb]{0.7,0.7,0.7}{Action}         & \colorbox[rgb]{0.7,0.7,0.7}{C}                     \\
2           & F            & 33           & 124.15               & 1.66                & N               & Adventure             & C                     \\
3           & F            & 57           & 65.13                & 1.80                & N               & Adventure             & C                     \\
4           & F            & 39           & 58.71                & 1.74                & N               & Comedy          & E                     \\
5  & \colorbox[rgb]{0.7,0.7,0.7}{F}            & 39           & 67.41                & 1.56                & N               & \colorbox[rgb]{0.7,0.7,0.7}{Action}         & \colorbox[rgb]{0.7,0.7,0.7}{C}                     \\
6  & \colorbox[rgb]{0.7,0.7,0.7}{F}            & 47           & 67.19                & 1.79                & Y               & \colorbox[rgb]{0.7,0.7,0.7}{Action}         & \colorbox[rgb]{0.7,0.7,0.7}{B}                     \\
7           & M            & 58           & 42.95                & 1.48                & N               & Action         & A                     \\
8  & \colorbox[rgb]{0.9,0.9,0.9}{M}            & 17           & 109.52               & \colorbox[rgb]{0.9,0.9,0.9}{1.62}                & \colorbox[rgb]{0.9,0.9,0.9}{N}               & Action         & \colorbox[rgb]{0.9,0.9,0.9}{C}                     \\
9          & F            & 48           & 58.07                & 1.50                & N               & Drama          & D                     \\
10  & \colorbox[rgb]{0.9,0.9,0.9}{M}            & 43           & 46.69                & \colorbox[rgb]{0.9,0.9,0.9}{1.61}                & \colorbox[rgb]{0.9,0.9,0.9}{N}               & Adventure             & \colorbox[rgb]{0.9,0.9,0.9}{B}                     \\
11  & \colorbox[rgb]{0.9,0.9,0.9}{M}            & 55           & 85.38                & \colorbox[rgb]{0.9,0.9,0.9}{1.54}                & \colorbox[rgb]{0.9,0.9,0.9}{N}               & Drama          & \colorbox[rgb]{0.9,0.9,0.9}{C}                     \\
12          & M            & 34           & 39.77                & 1.70                & N               & Action         & B                     \\
13          & M            & 51           & 55.72                & 1.93                & Y               & Drama          & B                     \\
14  & \colorbox[rgb]{0.7,0.7,0.7}{F}            & 47           & 57.10                & 1.51                & N               & \colorbox[rgb]{0.7,0.7,0.7}{Action}         & \colorbox[rgb]{0.7,0.7,0.7}{C}                     \\
15          & M            & 38           & 54.01                & 1.85                & Y               & Drama          & C                     \\
16  & \colorbox[rgb]{0.9,0.9,0.9}{M}            & 45           & 73.10                & \colorbox[rgb]{0.9,0.9,0.9}{1.59}                & \colorbox[rgb]{0.9,0.9,0.9}{N}               & Drama          & \colorbox[rgb]{0.9,0.9,0.9}{C} \\
\bottomrule
\end{tabular}
\label{tab:maDataEx1}
\end{table}

Notice that categorical attributes are discrete and finite entities. The domain of a discrete attribute can be represented by a set of symbols without any loss of information. In the ordinal case, a set of integer values obeying a bijective mapping is a straightforward choice. It is the same in the nominal case if we are just focusing on detecting if the attribute value is coincident or not. Other more elaborate or application-dependent mappings are certainly admissible.

Given this mapping for the categorical attributes, our generalized definition of a CVC bicluster in Subsection~\ref{ssec:ma} allows an immediate manipulation of mixed-attribute datasets. We can consider one particular $\epsilon$ per attribute, and every time that a nominal attribute is being manipulated in a specific column of the mixed-attribute dataset, $\epsilon$ should be taken as zero. On the other hand, for categorical attributes exhibiting an ordinal relation, a suitable integer value should be adopted for $\epsilon$ (it will depend on what the user wants to accept as being part of the same group). For an integer or real-valued attribute, an option is to determine the value of $\epsilon$ using a binning algorithm (for instance, we can use the Scott's rule if the attribute is close to being normally distributed). Notice that we are free to choose different binning algorithms for different numerical attributes.

To illustrate this explanation, Table~\ref{tab:maDataEx1} has two highlighted submatrices that are examples of perturbed CVC biclusters. One of the biclusters of Table~\ref{tab:maDataEx1} could be described as the following quantitative itemset: \{Sex\{M\}, Height[1.54,1.62], Smoker\{N\}, SocialClass\{B,C\}\}. For simplicity, we omitted the symbol $\in$.

\section{Comparison between enumerative biclustering approaches with a priori and online partitioning}
\label{sec:appendix_class1x2}

\setcounter{table}{0}

Using a didactic example, by the first time in the literature we are going to illustrate the practical effects of the conceptual differences between enumerative biclustering approaches with a priori partitioning (named here Class1) and online partitioning (named here Class2).

Table~\ref{tab:ec1prep} shows an example of a dataset and its preprocessing for the usage of traditional enumerative algorithms, such as In-Close or LCM. Table~\ref{tab:ec1prep}(a) shows the original dataset. Its attributes are integers uniformly distributed between 1 and 15. We purposely chose integer attributes to highlight that the partitioning of the original dataset may be necessary even when it is already discrete. In some cases, the partitioning is mandatory, for instance when handling real-valued attributes.

Table~\ref{tab:ec1prep}(b) shows the dataset of Table~\ref{tab:ec1prep}(a) after partitioning with equal-width binning (bin size equal to 5). As the attributes of the dataset of Table~\ref{tab:ec1prep}(a) are uniformly distributed, equal-width binning is the most indicated option to perform partitioning. The itemized dataset is shown in Table~\ref{tab:ec1prep}(c), which is equivalent to the discrete dataset of Table~\ref{tab:ec1prep}(b), i.e., they have exactly the same information. The number of columns of the itemized dataset depends on the number of bins used in the partitioning. In our example, we used 3 bins for each of the 4 attributes, so the itemized dataset has 12 attributes. Having these two equivalent matrices, the problem of enumerating all maximal perfect CVC biclusters in the partitioned matrix of Table~\ref{tab:ec1prep}(b) is equivalent to the problem of enumerating all maximal CTV biclusters of ones in the itemized matrix of Table~\ref{tab:ec1prep}(c).

Table~\ref{tab:bicsT1b} lists all maximal perfect CVC biclusters from Table~\ref{tab:ec1prep}(b) with at least two rows and one column, and Table~\ref{tab:bicsT1c} lists all maximal perfect CTV biclusters of ones from Table~\ref{tab:ec1prep}(c) with at least two rows and one column. The biclusters of Tables~\ref{tab:bicsT1b} and \ref{tab:bicsT1c} have a one-to-one correspondence. For instance, the bicluster $B_4=(\{1,2,4\},\{3,6\})$ of Table~\ref{tab:bicsT1c} and the bicluster $B_4=(\{1,2,4\},\{1,2\})$ of Table~\ref{tab:bicsT1b} are equivalent. Thus, having the biclustering solution based on the itemized dataset, it is easy to get the corresponding biclusters based on the partitioned / original dataset. Notice that the biclusters are perfect w.r.t. the partitioned dataset, but they can be perturbed w.r.t. the original dataset.

Since we have enumerative biclustering algorithms such as RIn-Close\_CVCP (which is specialized in enumerating all maximal and perfect CVC biclusters), we can obtain the biclustering solution of Table~\ref{tab:bicsT1b} by directly applying RIn-Close\_CVCP in the dataset of Table~\ref{tab:ec1prep}(b), thus avoiding the necessity of (1) itemization and (2) mapping the biclustering solution to the partitioned / original dataset. However, the worst-case time-complexity of RIn-Close\_CVCP is slightly worse than the worst-case time-complexity of some efficient enumerative biclustering algorithms for binary datasets. For instance, the worst-case time-complexity of In-Close algorithms and RIn-Close\_CVCP are, respectively, $O(qnm^2)$ and $O(qnm(\log n + m))$. At the same time, notice that the user is not totally comfortable to choose the number of bins when using itemization since it dictates the number $m$ of columns of the itemized dataset. Working with the original value of $m$, a reduced value when compared with the one achieved by itemization, RIn-Close\_CVCP is systematically more parsimonious in terms of computational resources.

Other important aspect is that the biclustering solutions shown in Tables~\ref{tab:bicsT1b} and \ref{tab:bicsT1c} are complete and all their biclusters are maximal when we consider the partitioned / itemized datasets (Tables~\ref{tab:ec1prep}(b) and \ref{tab:ec1prep}(c), respectively). It is not true when we consider the original dataset (Table~\ref{tab:ec1prep}(a)) and $\epsilon = 5$ as the maximum perturbation allowed in the partitioning. Table~\ref{tab:bicsT1a} lists all maximal CVC biclusters from Table~\ref{tab:ec1prep}(a) with maximum perturbation $\epsilon = 5$ and at least two rows and one column. Note that the biclustering solutions of Tables~\ref{tab:bicsT1b} and \ref{tab:bicsT1c} are contained in the biclustering solution of Table~\ref{tab:bicsT1a}. For instance: bicluster $B_8=(\{1,2,4,5\},\{1,2\})$ of Table~\ref{tab:bicsT1a} is a maximal version of bicluster $B_4=(\{1,2,4\},\{1,2\})$ of Table~\ref{tab:bicsT1b}, bicluster $B_5=(\{5,7,8\},\{1,3\})$ of Table~\ref{tab:bicsT1a} is a maximal version of bicluster $B_3=(\{5,8\},\{1\})$ of Table~\ref{tab:bicsT1b}, bicluster $B_6=(\{5,7\},\{1,2,3,4\})$ of Table~\ref{tab:bicsT1a} is not listed in Table~\ref{tab:bicsT1b}, and so on. In fact, in this example, only 4 biclusters are in their maximal versions and 7 biclusters are missed when the a priori partition is applied.

A proposal to mitigate these drawbacks of Class1 algorithms 
is the assignment of multiple items for values near the boundary of the bins \cite{HenriquesMadeira2014}. Table~\ref{tab:ec1prep}(d) shows an example of itemization with multiple items assignments. Table~\ref{tab:bicsT1d} lists all maximal perfect CTV biclusters of ones from Table~\ref{tab:ec1prep}(d) with at least two rows and one column. Notice that this solution is not complete when compared to the one in Table~\ref{tab:bicsT1a}. 
Moreover, we have an additional shortcoming: the enumeration is non-redundant from the point of view of the itemized dataset
, but it is redundant from the point of view of the partitioned / original matrix. For instance, the biclusters $B_{20}=(\{1,3,5\},\{8,11,12\})$ and $B_{21}=(\{3,5\},\{8,9,11,12\})$ of Table~\ref{tab:bicsT1d} are mapped, respectively, to the biclusters $(\{1,3,5\},\{3,4\})$ and $(\{3,5\},\{3,4\})$ considering the original / partitioned datasets (Tables~\ref{tab:ec1prep}(a) and \ref{tab:ec1prep}(b), respectively). The bicluster $(\{3,5\},\{3,4\})$ is clearly non-maximal, being a redundant information. Enumerative solutions usually have a huge number of biclusters, so the post-processing of the solution to eliminate redundancy can be very costly.

\begin{table}[!htb]
\caption{Example of a dataset and its preprocessing for the usage of traditional FPM / FCA / GT enumerative algorithms.}
\begin{minipage}{.5\linewidth}
\centering
\small
\captionof{subtable}{Original dataset (its attributes are integers uniformly distributed between 1 and 15).}
\begin{tabular}{r|rrrr}
\hline
&$y_1$&$y_2$&$y_3$&$y_4$\\
\hline
$x_1$&13&15&7&11\\
$x_2$&14&15&14&12\\
$x_3$&2&3&12&12\\
$x_4$&14&15&15&6\\
$x_5$&10&15&10&10\\
$x_6$&2&8&1&3\\
$x_7$&5&13&13&11\\
$x_8$&9&3&15&1
\label{tab:ec1prep_a}
\end{tabular}
\end{minipage}%
\begin{minipage}{.5\linewidth}
\centering
\small
\captionof{subtable}{Dataset of Table~\ref{tab:ec1prep}(a) after partitioning with equal-width binning (bin size = 5).}
\begin{tabular}{r|rrrr}
\hline
&$y_1$&$y_2$&$y_3$&$y_4$\\
\hline
$x_1$&3&3&2&3\\
$x_2$&3&3&3&3\\
$x_3$&1&1&3&3\\
$x_4$&3&3&3&2\\
$x_5$&2&3&2&2\\
$x_6$&1&2&1&1\\
$x_7$&1&3&3&3\\
$x_8$&2&1&3&1
\label{tab:ec1prep_b}
\end{tabular}
\end{minipage}
\begin{minipage}{.5\linewidth}
\centering
\small
\setlength\tabcolsep{3pt} 
\captionof{subtable}{Dataset of Table~\ref{tab:ec1prep}(b) after itemization.}
\begin{tabular}{l|ccc|ccc|ccc|ccc}
\hline
& \multicolumn{3}{c}{$y1$} & \multicolumn{3}{c}{$y2$} & \multicolumn{3}{c}{$y3$} & \multicolumn{3}{c}{$y4$} \\
& 1 & 2 & 3 & 1 & 2 & 3 & 1 & 2 & 3 & 1 & 2 & 3\\
\hline
$x_1$& & &$\times$& & &$\times$& &$\times$& & & &$\times$\\
$x_2$& & &$\times$& & &$\times$& & &$\times$& & &$\times$\\
$x_3$&$\times$& & &$\times$& & & & &$\times$& & &$\times$\\
$x_4$& & &$\times$& & &$\times$& & &$\times$& &$\times$&\\
$x_5$& &$\times$& & & &$\times$& &$\times$& & &$\times$&\\
$x_6$&$\times$& & & &$\times$& &$\times$& & &$\times$& &\\
$x_7$&$\times$& & & & &$\times$& & &$\times$& & &$\times$\\
$x_8$& &$\times$& &$\times$& & & & &$\times$&$\times$& &
\label{tab:ec1prep_c}
\end{tabular}
\end{minipage}
\begin{minipage}{.5\linewidth}
\centering
\small
\setlength\tabcolsep{3pt} 
\captionof{subtable}{Dataset of Table~\ref{tab:ec1prep}(b) after itemization with multiple item assignments. The symbol $+$ indicates the new assignments when compared with Table~\ref{tab:ec1prep}(c).}
\begin{tabular}{l|ccc|ccc|ccc|ccc}
\hline
& \multicolumn{3}{c}{$y1$} & \multicolumn{3}{c}{$y2$} & \multicolumn{3}{c}{$y3$} & \multicolumn{3}{c}{$y4$} \\
& 1 & 2 & 3 & 1 & 2 & 3 & 1 & 2 & 3 & 1 & 2 & 3\\
\hline
$x_1$& & &$\times$& & &$\times$&$+$&$\times$& & &$+$&$\times$\\
$x_2$& & &$\times$& & &$\times$& & &$\times$& &$+$&$\times$\\
$x_3$&$\times$& & &$\times$& & & &$+$&$\times$& &$+$&$\times$\\
$x_4$& & &$\times$& & &$\times$& & &$\times$&$+$&$\times$&\\
$x_5$& &$\times$&$+$& & &$\times$& &$\times$&$+$& &$\times$&$+$\\
$x_6$&$\times$& & & &$\times$& &$\times$& & &$\times$& &\\
$x_7$&$\times$&$+$& & & &$\times$& & &$\times$& &$+$&$\times$\\
$x_8$& &$\times$&$+$&$\times$& & & & &$\times$&$\times$& &
\end{tabular}
\label{tab:ec1prep_d}
\end{minipage}
\label{tab:ec1prep}
\end{table}

\begin{table}[!htb]
\caption{List of all maximal perfect CVC biclusters, $B_i = (I_i, J_i)$, from Table~\ref{tab:ec1prep}(b) with at least two rows and one column.}
\centering
\small
\begin{tabular}{l|l|l}
\toprule
$B_1=(\{3,6,7\},\{1\})$ & $B_2=(\{3,7\},\{1,3,4\})$ & $B_3=(\{5,8\},\{1\})$\\
$B_4=(\{1,2,4\},\{1,2\})$ & $B_5=(\{2,4\},\{1,2,3\})$ & $B_6=(\{1,2\},\{1,2,4\})$\\
$B_7=(\{3,8\},\{2,3\})$ & $B_8=(\{1,2,4,5,7\},\{2\})$ & $B_9=(\{1,5\},\{2,3\})$\\
$B_{10}=(\{2,4,7\},\{2,3\})$ & $B_{11}=(\{2,7\},\{2,3,4\})$ & $B_{12}=(\{4,5\},\{2,4\})$\\
$B_{13}=(\{1,2,7\},\{2,4\})$ & $B_{14}=(\{2,3,4,7,8\},\{3\})$ & $B_{15}=(\{2,3,7\},\{3,4\})$\\
$B_{16}=(\{6,8\},\{4\})$ & $B_{17}=(\{1,2,3,7\},\{4\})$ &\\
\bottomrule
\end{tabular}
\label{tab:bicsT1b}
\end{table}

\begin{table}[!htb]
\caption{List of all maximal perfect CTV biclusters of ones, $B_i = (I_i, J_i)$, from Table~\ref{tab:ec1prep}(c) with at least two rows and one column.}
\centering
\small
\begin{tabular}{l|l|l}
\toprule
$B_1=(\{3,6,7\},\{1\})$ & $B_2=(\{3,7\},\{1,9,12\})$ & $B_3=(\{5,8\},\{2\})$\\
$B_4=(\{1,2,4\},\{3,6\})$ & $B_5=(\{2,4\},\{3,6,9\})$ & $B_6=(\{1,2\},\{3,6,12\})$\\
$B_7=(\{3,8\},\{4,9\})$ & $B_8=(\{1,2,4,5,7\},\{6\})$ & $B_9=(\{1,5\},\{6,8\})$\\
$B_{10}=(\{2,4,7\},\{6,9\})$ & $B_{11}=(\{2,7\},\{6,9,12\})$ & $B_{12}=(\{4,5\},\{6,11\})$\\
$B_{13}=(\{1,2,7\},\{6,12\})$ & $B_{14}=(\{2,3,4,7,8\},\{9\})$ & $B_{15}=(\{2,3,7\},\{9,12\})$\\
$B_{16}=(\{6,8\},\{10\})$ & $B_{17}=(\{1,2,3,7\},\{12\})$ &\\
\bottomrule
\end{tabular}
\label{tab:bicsT1c}
\end{table}

\begin{table}[!htb]
\caption{List of all maximal CVC biclusters, $B_i = (I_i, J_i)$, from Table~\ref{tab:ec1prep}(a) with maximum perturbation $\epsilon = 5$ and at least two rows and one column.}
\centering
\small
\begin{tabular}{l|l|l}
\toprule
$B_1=(\{3,6,7\},\{1\})$ & $B_2=(\{3,6\},\{1,2\})$ & $B_3=(\{6,7\},\{1,2\})$\\
$B_4=(\{3,7\},\{1,3,4\})$ & $B_5=(\{5,7,8\},\{1,3\})$ & $B_6=(\{5,7\},\{1,2,3,4\})$\\
$B_7=(\{1,2,4,5,8\},\{1\})$ & $B_8=(\{1,2,4,5\},\{1,2\})$ & $B_9=(\{1,5\},\{1,2,3,4\})$\\
$B_{10}=(\{2,4,5\},\{1,2,3\})$ & $B_{11}=(\{4,5\},\{1,2,3,4\})$ & $B_{12}=(\{2,5\},\{1,2,3,4\})$\\
$B_{13}=(\{1,4,5\},\{1,2,4\})$ & $B_{14}=(\{1,2,5\},\{1,2,4\})$ & $B_{15}=(\{2,4,5,8\},\{1,3\})$\\
$B_{16}=(\{4,8\},\{1,3,4\})$ & $B_{17}=(\{3,6,8\},\{2\})$ & $B_{18}=(\{3,8\},\{2,3\})$\\
$B_{19}=(\{6,8\},\{2,4\})$ & $B_{20}=(\{1,2,4,5,7\},\{2\})$ & $B_{21}=(\{2,4,5,7\},\{2,3\})$\\
$B_{22}=(\{4,5,7\},\{2,3,4\})$ & $B_{23}=(\{2,5,7\},\{2,3,4\})$ & $B_{24}=(\{1,4,5,7\},\{2,4\})$\\
$B_{25}=(\{1,2,5,7\},\{2,4\})$ & $B_{26}=(\{1,3,5\},\{3,4\})$ & $B_{27}=(\{2,3,4,5,7,8\},\{3\})$\\
$B_{28}=(\{2,3,5,7\},\{3,4\})$ & $B_{29}=(\{4,6,8\},\{4\})$ & $B_{30}=(\{1,2,3,5,7\},\{4\})$\\
\bottomrule
\end{tabular}
\label{tab:bicsT1a}
\end{table}

\begin{table}[!htb]
\caption{List of all maximal perfect CTV biclusters of ones, $B_i = (I_i, J_i)$, from Table~\ref{tab:ec1prep}(d) with at least two rows and one column.}
\centering
\small
\begin{tabular}{l|l|l}
\toprule
$B_1=(\{3,6,7\},\{1\})$ & $B_2=(\{3,7\},\{1,9,11,12\})$ & $B_3=(\{5,7,8\},\{2,9\})$\\
$B_4=(\{5,8\},\{2,3,9\})$ & $B_5=(\{5,7\},\{2,6,9,11,12\})$ & $B_6=(\{1,2,4,5,8\},\{3\})$\\
$B_7=(\{1,2,4,5\},\{3,6,11\})$ & $B_8=(\{1,5\},\{3,6,8,11,12\})$ & $B_9=(\{2,4,5\},\{3,6,9,11\})$\\
$B_{10}=(\{2,5\},\{3,6,9,11,12\})$ & $B_{11}=(\{1,2,5\},\{3,6,11,12\})$ & $B_{12}=(\{2,4,5,8\},\{3,9\})$\\
$B_{13}=(\{4,8\},\{3,9,10\})$ & $B_{14}=(\{3,8\},\{4,9\})$ & $B_{15}=(\{1,2,4,5,7\},\{6,11\})$\\
$B_{16}=(\{2,4,5,7\},\{6,9,11\})$ & $B_{17}=(\{2,5,7\},\{6,9,11,12\})$ & $B_{18}=(\{1,2,5,7\},\{6,11,12\})$\\
$B_{19}=(\{1,6\},\{7\})$ & $B_{20}=(\{1,3,5\},\{8,11,12\})$ & $B_{21}=(\{3,5\},\{8,9,11,12\})$\\
$B_{22}=(\{2,3,4,5,7,8\},\{9\})$ & $B_{23}=(\{2,3,4,5,7\},\{9,11\})$ & $B_{24}=(\{2,3,5,7\},\{9,11,12\})$\\
$B_{25}=(\{4,6,8\},\{10\})$ & $B_{26}=(\{1,2,3,4,5,7\},\{11\})$ & $B_{27}=(\{1,2,3,5,7\},\{11,12\})$\\
\bottomrule
\end{tabular}
\label{tab:bicsT1d}
\end{table}


\singlespacing

\section*{References}

{\footnotesize
\bibliography{tese}

\vspace{5pt}

\textbf{Rosana Veroneze} is a postdoctoral researcher at the Department of Computer Engineering and Industrial Automation, School of Electrical and Computer Engineering, University of Campinas (Unicamp). Her research interests include computational intelligence, data mining and machine learning areas.

\vspace{5pt}

\textbf{Fernando J. Von Zuben} is a Full Professor at the Department of Computer Engineering and Industrial Automation, School of Electrical and Computer Engineering, University of Campinas (Unicamp). The main topics of his research are computational intelligence, bioinspired computing, multivariate data analysis, and machine learning.
}
\end{document}